\documentclass{article}

\usepackage[final]{neurips_2022}
\usepackage{amsmath,amsthm,verbatim,amssymb,amsfonts,amscd}
\usepackage{graphics}
\usepackage{centernot}
\usepackage{caption}
\usepackage{subcaption}
\usepackage{helvet}  
\usepackage{courier}  
\usepackage{url}  
\usepackage{graphicx}  
\usepackage{multirow}
\usepackage{color}
\usepackage{MnSymbol}
\usepackage{makecell}
\usepackage{arydshln}
\usepackage[dvipsnames]{xcolor}
\usepackage{caption} 
\usepackage{natbib}
\usepackage{textcomp}
\usepackage{wrapfig}
\usepackage{booktabs}
\usepackage{algorithm}
\usepackage{algorithmic}
\usepackage{bbm}
\usepackage{csquotes}
\usepackage{paralist}
\usepackage{caption}
\usepackage{subcaption}
\usepackage[colorlinks,linkcolor=blue,citecolor=blue]{hyperref}

\usepackage[symbol, bottom]{footmisc}

\renewcommand*{\thefootnote}{\fnsymbol{footnote}}

\theoremstyle{plain}
\newtheorem{theorem}{Theorem}
\newtheorem{corollary}{Corollary}
\newtheorem{lemma}{Lemma}
\newtheorem*{remark}{Remark}
\newtheorem{proposition}{Proposition}

\newcommand{\E}{\mathbb{E}}

\DeclareMathOperator*{\argmin}{arg\,min} 

\title{Posterior Collapse of a Linear Latent Variable Model}
\author{%
  Zihao Wang$^*$ \\
  Department of CSE\\
  HKUST\\
  \And
  Liu Ziyin\thanks{Equal contribution.} \\
  Department of Physics\\ 
  The University of Tokyo\\
}

\begin{document}

\maketitle

\begin{abstract}
    This work identifies the existence and cause of a type of posterior collapse that frequently occurs in the Bayesian deep learning practice. For a general linear latent variable model that includes linear variational autoencoders as a special case, we precisely identify the nature of posterior collapse to be the competition between the likelihood and the regularization of the mean due to the prior. Our result suggests that posterior collapse may be related to neural collapse and dimensional collapse and could be a subclass of a general problem of learning for deeper architectures.
\end{abstract}

\renewcommand*{\thefootnote}{\arabic{footnote}}
\setcounter{footnote}{0}

\vspace{-2mm}
\section{Introduction}
\vspace{-1mm}

Bayesian approaches to deep learning have attracted much attention because they allow for a more principled treatment of inference and uncertainty estimation \citep{mackay1992bayesian, neal2012bayesian, wang2020survey,jiang2020generative,zhao2021assessment,liu2021relational}. One long-standing and unresolved problem for the Bayesian deep learning practice is the problem of posterior collapse, where the posterior distribution of the learned latent variables partially or completely collapses with the prior \citep{bowman2015generating, huang2018improving, lucas2019don, razavi2019preventing, kingma2016improved, wang2021posterior}. Up to now, the study of the nature of the cause of the posterior collapse problem has been limited. There are two main challenges that prevent our understanding of the problem: (1) posterior collapses mainly occur in deep learning, and the landscape of deep neural networks is hard to understand in general; (2) the use of approximate loss functions such as the evidence lower bound (ELBO) complicates the problem.

Consider a problem where one wants to model the data distribution $p(x,y)$ through a latent variable $z$. The evidence lower bound (ELBO) loss function reads
\begin{equation}\label{eq: intro loss}
    \underbrace{\E_{x,y}[- \mathbb{E}_{q(z|x)}\log(p(y|z))]}_{\ell_{rec}} + \underbrace{\E_x[D_{KL}(q(z|x) \| p(z))]}_{\ell_{KL}},
\end{equation}
where $q$ is the approximate distribution, we rely on to approximate the true distribution $p$. This loss is more general than the standard ELBO for variational autoencoders (VAE) \citep{kingma2013auto}. Meanwhile, it can be seen as the simplest type of loss for a conditional VAE \citep{sohn2015learning}, where one aims to model a conditional distribution $p(y|x)$. The distribution $p(z)$ is the prior distribution of the latent variable $z$ and is often a low-complexity distribution such as a zero-mean unit-variance Gaussian. This loss function thus has a clean interpretation as the sum of a prediction accuracy term (the first term $\ell_{rec}$) that encourages better prediction accuracy and a complexity term (the second term $\ell_{KL}$) that encourages a simpler solution. Learning under this loss function proceeds by balancing the prediction error and the model simplicity. Moreover, learning under this loss function has also been used as one of the primary theoretical models in neuroscience \citep{friston2009free}, and its understanding may also help advance theoretical neuroscience. This work provides an in-depth study of the posterior collapse problem of Eq.~\eqref{eq: intro loss}, when the decoder $q(y|z)$ and encoder $q(z|x)$ are each parametrized by a linear model.

Specifically, our contributions include: 
\begin{compactitem}
    \item we find the global minima of a general linear latent variable model that includes the linear VAE as a special case under the Objective~\eqref{eq: intro loss};
    \item we find the precise condition when posterior collapse occurs, where the global minimum is the origin;
    \item we pinpoint the cause of the posterior collapse to be the excessively strong regularization effect on the \textit{mean} of the latent variables due to the prior.
\end{compactitem}
To the best of our knowledge, our work is the first to pinpoint the cause of the posterior collapse problem. This work is organized as follows. The next section discusses the previous literature. Section~\ref{sec:problem-setting} describes the theoretical problem setting. Section~\ref{sec: main results} presents our main technical results and analyzes them in detail. Section~\ref{sec: experiment} presents numerical examples. The last section concludes this work and points to the remaining open problems. The Appendix~\ref{app:bias-term} investigates the effect of the bias term, Appendix~\ref{app sec: data dependent encoder variance} details the effect of a data-dependent encoder variance, and Appendix~\ref{app sec: learnable decoder variance} treats the case of a learnable decoder variance.

\vspace{-2mm}
\section{Related Works}
\vspace{-1mm}

\textit{Approximate Bayesian Deep Learning}. Bayesian deep learning in general and VAE training, in particular, rely heavily on approximate methods such as the ELBO objective because the exact probabilities are intractable. The connection of approximate Bayesian learning and probabilistic PCA (pPCA) has been extensively studied \citep{nakajima2010implicit, nakajima2013global, nakajima2015condition, lucas2019don}.

\textit{Causes of Posterior Collapse}. Earlier touches on the problem tend to attribute the cause of posterior collapse to the use of approximate methods, namely, to the use of the ELBO \citep{bowman2015generating, huang2018improving, razavi2019preventing}. Another line of work attributes the cause to the high capacity of modern neural networks \citep{alemi2018fixing, ziyin2022stochastic}. However, \cite{lucas2019don} showed that for a simplified linear model, the ELBO is not the cause of posterior collapse because the posterior collapse exists even in the exact posterior. It also implies that the posterior collapse is not due to the high capacity of the models because linear models have a limited capacity. \cite{lucas2019don} then suggested that making the decoder variance learnable can fix the collapse problem and that an unlearnable decoder variance is the cause of the posterior collapse. However, our results show that this is not the case: for both learnable and unlearnable decoder variance, there exist situations where a collapse happens or does not happen, which implies that the learnability of the decoder variance does not have a causal relation with posterior collapse, nor is it sufficient to fix the problem (Section~\ref{app sec: learnable decoder variance}). In terms of the problem setting, ours is also more general than \cite{lucas2019don} because our result (1) applies to general latent variable models (one example being the conditional VAE \citep{sohn2015learning}) and (2) considers the case of $\beta$-VAE with a general $\beta$ when the decoder variance is learned. An important implication of our work is that posterior collapses can be a ubiquitous problem for deep-learning-based latent-variable models (not just unique to autoencoding models) and that they share a common cause. Meanwhile, \citep{lucke2020evidence} shows that posterior collapse can happen due to the tradeoff between the decoding performance and the decoding entropy. \citep{shekhovtsov2022vae} demonstrated the relationship between model consistency and posterior collapse and suggested that a proper choice of data processing or architecture may alleviate collapse.


\textit{Linear Networks}. Deep linear nets have been extensively used to understand the landscape of nonlinear networks. For example, linear regressors are shown to be relevant for understanding the generalization behavior of modern overparametrized networks \citep{hastie2019surprises}. \cite{saxe2013exact} used a two-layer linear network to understand the dynamics of learning nonlinear networks. The linear nets are the same as a linear regression model in terms of expressivity. However, the loss landscape is highly complicated due to depth. \citep{kawaguchi2016deep,hardt2016identity,laurent2018deep, ziyin2022exact}. Our work essentially studies the loss landscape of linear networks. While each encoder and decoder we use consists of a single linear layer, they effectively constitute a two-layer linear network when trained together.
  
\vspace{-2mm}
\section{Problem Setting}\label{sec:problem-setting}
\vspace{-2mm}
We consider a general linear latent variable model with input space $x\in \mathbb{R}^{D_0}$, latent space $z\in \mathbb{R}^{d_1}$, and target space $y\in \mathbb{R}^{d_2}$. In general, $y=f(x)$ is an arbitrary function of $x$. When the target $y$ is identical to the input $x$, it reduces to the standard VAE. The VAE formalism assumes that there is an intermediate ``latent variable" $z$ that captures the data generation process. In the main text, the encoder and decoder are linear transformations without bias terms, and the learnable bias is treated in Appendix~\ref{app:bias-term}, which shows that the effect of the bias terms is equivalent to centering both the input and target to be zero-mean ($x\to x- \E[x]$, $y\to y-\E[y]$). Incorporating the bias terms thus does not affect the main results. 
Specifically, the encoder is defined as $z = W^\top x + \epsilon$, where $\epsilon \sim \mathcal{N}(0, \Sigma)$ is the noise distribution introduced by the reparameterization trick where the variance matrix $\Sigma = {\rm diag}(\sigma_{1}^2,...,\sigma_{d_1}^2)$ is assumed to be diagonal and independent from $x$. The decoder parametrizes the distribution $p(y|z) = \mathcal{N}(U z, \eta_{\rm dec}^2I)$, where the variance $\eta_{\rm dec}^2I$ is to be isotropic and input-independent. In alignment with the standard practice, we also assume the prior distribution of latent variable $p(z) = \mathcal{N}(0, \eta_{\rm enc}^2 I)$ is an isotropic normal distribution, and the encoding variances matrix $\Sigma$ is learned from the data distribution while $\eta_{\rm dec}^2$ is not learnable. Lastly, we weigh the KL term by a coefficient $\beta$, which is a common practice in VAE training \citep{higgins2016beta}. Hence, the objective of such a linear model reads,\footnote{We note that $\E_x$ is the expectation over the training set. Also, we use the subscript "VAE" because the model can be seen as a conditional VAE, even though it may be more proper to call it a "general latent variable model."}
\begin{align}
    & L_{\rm VAE}(U, W, \Sigma) \\
    & = \mathbb{E}_{x}[- \mathbb{E}_{q(z|x)}\log(p(y|z)) + \beta D_{KL}(q(z|x) \| p(z;\eta^2_{\rm enc})) ] \\
    & = \frac{1}{2\eta_{\rm dec}^2} \mathbb{E}_{x,\epsilon} \left[ \|U (W^\top x+ \epsilon) - y\|_2^2 + \beta \frac{\eta_{\rm dec}^2}{\eta_{\rm enc}^2} \|W^\top x\|^2 \right] + \sum_{i=1}^{d_1} \frac{\beta}{2} \left(\frac{\sigma_{i}^2}{\eta^2_{\rm enc}} - 1 - \log \frac{\sigma_{i}^2}{\eta_{\rm enc}^2}\right)~\label{eq:raw_VAE_loss}\\
    & = \frac{1}{2\eta_{\rm dec}^2}\left[ \mathbb{E}_{x} \|U W^\top x - y\|_2^2 +  {\rm Tr}(U \Sigma U^\top) + \underbrace{\beta \frac{\eta_{\rm dec}^2}{\eta_{\rm enc}^2} {\rm Tr}(W^\top AW)}_{\ell_{mean}}\right]+ \underbrace{\sum_{i=1}^{d_1} \frac{\beta}{2} \left(\frac{\sigma_{i}^2}{\eta^2_{\rm enc}} - 1 - \log \frac{\sigma_{i}^2}{\eta_{\rm enc}^2}\right)}_{\ell_{var}},\label{eq:general-vae-loss}
\end{align}
where $A := \mathbb{E}_{x} [x x^\top]$ is the second moment of the input data. Note that a crucial feature of the KL term is that it decomposes into two terms, one that regularizes the variance of $z$ ($\ell_{var}$) and another that regularizes the mean of $z$ ($\ell_{ mean}$). We will see that it is precisely the $\ell_{ mean}$ term that causes the posterior collapse. Eq.~\eqref{eq:general-vae-loss} has ignored the partition function of the decoder because we treat $\eta_{\rm dec}$ as a constant. We study the case of a learnable $\eta_{\rm dec}$ in Section~\ref{app sec: learnable decoder variance}. In comparison to the previous works \citep{lucas2019don} that have treated the case of a learnable $\eta_{\rm dec}$, our result is more general because our result also considers the effect of $\beta$ and allows for the case $d_1 \geq d_2$. It is also worth commenting on the difference between this setting and that of the pPCA setting \citep{nakajima2015condition}: (1) the effect of $\beta$ is included in the VAE loss, (2) the prior of VAE is over the latent variable, whereas pPCA has it over the model parameters, and (3) the model can be overparametrized ($d_1\geq d_0$).


\textbf{Notation}. To summarize, we use $x, y$, and $z$ to denote the input variable, latent variable, and target variable, respectively. $\E_x$ denotes the expectation over the training set. 
$A:= \E_x[xx^T]$ is the second moment matrix of the input $x$.
$A$ is thus positive semidefinite by definition. The eigenvalue decomposition of $A$ is $A = P_A \Phi P_A^\top$, where $\Phi = {\rm diag} (\phi_1, \cdots, \phi_{d_0})$ is the diagonal matrix for all $ d_0 \leq D_0$ positive eigenvalues and $P_A = [p_1, \cdots, p_{d_0}]$ are matrices by concatenating $d_0$ eigenvectors $p_i\in \mathbb{R}^{D_0}$.
$W$ and $U$ are learnable linear transformation matrices for the linear encoding and decoding processes. $\Sigma$ is the learnable diagonal latent variance matrix for encoder with diagonal entries $\sigma_i$. $\eta_{\rm enc}$ is the standard deviation of the prior distribution $p(z)$. $\eta_{\rm dec}$ is the standard deviation of decoded samples. A frequently used quantity is a whitened and rotated $x$: $\tilde x := \Phi^{-\frac{1}{2}} P_A^\top x$. Note that this transformation can be inverted: $x = P_A \Phi^{\frac{1}{2}} \tilde x$. We see that $\mathbb{E}_x \tilde x \tilde x^\top = I$. Furthermore, we define $Z := \mathbb{E}_{\tilde x}[y\tilde x^\top] = \mathbb{E}_{x}[y x^\top P_A \Phi^{-\frac{1}{2}}]\in \mathbb{R}^{d_2\times d_0}$. Let $Z = F \Sigma_Z G^\top$ be the singular value decomposition of $Z$, where $F\in \mathbb{R}^{d_2\times d_2}$ and $G\in \mathbb{R}^{d_0\times d_0}$ are two orthogonal matrices. $\Sigma_Z \in \mathbb{R}^{d_2\times d_0}$ is a rectangular diagonal matrix with $d^* = \min(d_0, d_2)$ singular values of $Z$ in the non-increasing order, i.e., $\zeta_1 \geq \zeta_2 \geq \cdots \geq \zeta_{d^*}\geq 0$.


\section{Main Results}\label{sec: main results}
This section discusses the main results, whose proofs are presented in Appendix~\ref{app sec: proofs}. While $\Sigma$ is often a learnable parameter, we first assume that the KL term is sufficiently strong such that $\sigma_1 = \cdots = \sigma_{d_1} \approx \eta_{\rm enc}$ is close to the prior value. We then compare with the case when it is learnable, and this comparison reveals that an optimizable $\sigma_{i}$ is not essential to the posterior collapse problem.

\subsection{General Result}
In this section, we prove two results that will be useful for understanding the nature of the VAE training objective and will be useful for us to find the global minimum. We first show that the VAE objective is equivalent to a matrix factorization problem with a special type of regularization.


\begin{proposition}\label{proposition1}
Let $\tilde x := \Phi^{-\frac{1}{2}} P_A^\top x$, $Z:= E_{\tilde{
x}}[y\tilde{x}^{\top}]$, and
\begin{align}
    (U^*, V^*) := \argmin_{(U, V)} L(U, V) = \argmin_{(U, V)} \|UV^\top - Z\|^2_F + {\rm Tr}(U \Sigma U^\top) + \beta \frac{\eta_{\rm dec}^2}{\eta_{\rm enc}^2} \|V \|^2_F. \label{eq:LVAE_is_RSVD}
\end{align}
Given a fixed $\Sigma$, the minimizer of $L_{\rm VAE}(U, W, \Sigma)$ is
$(U^*, W^*)$, where $W^*$ is any solution of $ \Phi^{\frac{1}{2}} P_A^\top W =  V^* $.
\end{proposition}

\textit{Proof sketch}. The term $\ell_{var}$ is irrelevant to finding the optimal $U^*$ and $V^*$ when $\Sigma$ is fixed. Thus, the relevant objective can be obtained with the change of variables $\tilde x = \Phi^{-\frac{1}{2}} P_A^\top x$.  $\square$

The condition $\Phi^{\frac{1}{2}} P_A^\top W =  V^*$ shows that when the data is low-rank, each solution $(U^*,V^*)$ corresponds to a manifold of solutions in the original parameter space. The effective loss $L(U, V)$ can be compared with the regularized singular value decomposition problem~\citep{zheng2018regularized}. We see that the first term is the standard matrix factorization objective, while the second and third are unique regularization effects due to the VAE structure and the ELBO objective. In addition, the term $\Sigma$ in the second is the strength of the regularization for the norm of $U$, and a crucial difference with standard regularized matrix factorization is that $\Sigma$ is also a learnable matrix.

The next proposition finds, for any fixed $\Sigma$, the global minima $(U^*, V^*)$ of Eq.~\eqref{eq:LVAE_is_RSVD}. In particular, the learning is characterized by the learning of the singular values of $U$ and $V$.
\begin{proposition}~\label{proposition:min-Luv}
The optimal solution $(U^*, V^*)$ of $\min_{U, V} L(U, V)$ is given by
\begin{align}
    U^* &= F \Lambda P,\quad V^* = G \Theta P,
\end{align}
where $F \in \mathbb{R}^{d_2\times d_2}$ and $G\in \mathbb{R}^{d_0\times d_0}$ are orthogonal matrices derived by the SVD of $Z$, $P$ is an arbitrary orthogonal matrix in $\mathbb{R}^{d_1 \times d_1}$, and $\Lambda \in \mathbb{R}^{d_2\times d_1}$ and $\Theta \in \mathbb{R}^{d_0\times d_1}$ are rectangular diagonal matrices with the diagonal elements
\begin{align}
    \lambda_i = \sqrt{\max\left(0,\frac{\sqrt{\beta} \eta_{\rm dec}}{\sigma_i\eta_{\rm enc}} \left(\zeta_i - \frac{\sqrt{\beta}\sigma_i \eta_{\rm dec}}{\eta_{\rm enc}}\right)\right)}, \quad
    \theta_i = \sqrt{\max\left(0,\frac{\sigma_i \eta_{\rm enc}}{\sqrt{\beta} \eta_{\rm dec}} \left(\zeta_i - \frac{\sqrt{\beta}\sigma_i \eta_{\rm dec}}{\eta_{\rm enc}}\right)\right)}.
\end{align}
For convention, we let $\zeta_i = 0$ when $i > d^* = \min(d_0,d_2)$.
\end{proposition}


\textit{Proof sketch}. The optimal $V^*$ is a function of $U$ under the zero gradient condition. Thus, the objective reduces to single-variate with respect to $U$. The optimal $U^*$ is constructed by its SVD $U = Q\Lambda P$, where the optimal $Q^*$ and $\Lambda^*$ can be determined given the SVD of $Z$, and $P$ is left as a free orthogonal matrix. $V^*$ is determined once $U^*$ is obtained. $\square$

The readers are recommended to examine the form of the solutions closely. There are a few interesting features of the global minimum. One note that the sign of the term $\zeta_i - {\sqrt{\beta}\sigma_i\eta_{\rm dec}}/{\eta_{\rm enc}}$ is crucial, and can encourage the parameters $U$ and $V$ to be low-rank. Recall that $\sigma_i$ is the eigenvalue value of $ZZ^{\rm T}=E[y\tilde{x}^\top] E[y\tilde{x}^\top]^\top$, one can roughly identify $\zeta^2_{i}$ as the the strength of the alignment between the input $x$ and the target $y$. To see this, consider a simplified scenario where the target $y=\gamma Mx$ is a linear function of the input, where $\gamma$ is the overall strength of the signal and $||M||=1$ is a normalized orientation matrix, then
\begin{equation}
    ZZ^{\rm T} = \gamma^2 M^\top A M, 
\end{equation}
which is a positive semidefinite matrix. We see that there are two distinctive sources of contribution to the magnitude of the eigenvalues of $ZZ^{\rm T}$. Its eigenvalues are large if either the overall strength $\gamma$ is large or if the orientation matrix $M$ aligns well with the covariance of the input feature $A$.
Additionally, in the case of VAE, $\gamma M=I$, and $ZZ^{\rm T} = A$ is nothing but the covariance of input features, and $\zeta_i^2$ are the eigenvalues of $A$.

\subsection{Linear VAE without Learnable $\Sigma$ }\label{sec: nonlearnable}
We first consider the case where $\sigma_{i}$ is a constant that is completely determined by the prior: $\sigma_i = \eta_{\rm enc}$. This allows us to find a simplified form for the global minimum. The proof follows by plugging $\sigma_{i} = \eta_{\rm enc}$ into Proposition~\ref{proposition:min-Luv}.

\begin{theorem}\label{theo: main result wihout learnable sigma}
    Let $\sigma_i = \eta_{\rm enc}$ for all $i$. Then, the global minimum has
    \begin{align}
        \lambda_i = \sqrt{\max\left(0,\frac{\sqrt{\beta} \eta_{\rm dec}}{\eta_{\rm enc}^2} \left(\zeta_i - {\sqrt{\beta}\eta_{\rm dec}}\right)\right)}, \quad
        \theta_i = \sqrt{\max\left(0,\frac{\eta_{\rm enc}^2}{\sqrt{\beta} \eta_{\rm dec}} \left(\zeta_i - {\sqrt{\beta}\eta_{\rm dec}}\right)\right)}.
    \end{align}
\end{theorem}
There are three interesting observations of the global minimum. First of all, it depends crucially on the sign of $\zeta_i - {\sqrt{\beta}\eta_{\rm dec}}$ for all $i$. When the sign is negative for some $i$, the learned model becomes low-rank. Namely, some of the dimensions collapse with the prior. When the signs are all negative, we have a complete posterior collapse: both $U$ and $V$ are identically zero, so the latent variables have a distribution identical to the prior. A \textit{complete} posterior collapse happens if and only if $\max_i \zeta_i - {\sqrt{\beta}\eta_{\rm dec}} \leq 0$. A \textit{partial} posterior collapse happens if there exists $i$ such that $\zeta_i^2 - {\sqrt{\beta}\eta_{\rm dec}} \leq 0$. These two conditions give the precise conditions of posterior collapse in this scenario. This implies that having a sufficiently small $\beta$ will always prevent posterior collapse. The second observation is that the effect of $\beta$ is identical to that of $\eta_{\rm dec}$ because $\sqrt{\beta}$ and $\eta_{\rm dec}$ always appear together, and so one alternative way to fix posterior collapses is to use a sufficiently small $\eta_{\rm dec}$. From a Bayesian perspective, the latter method of tuning $\eta_{\rm dec}$ is better because $\eta_{\rm dec}$ comes directly from the (assumed) likelihood $p(x|\eta_{\rm dec})$. In contrast, the $\beta$ parameter is only an implementation technique that has obscure meaning in the Bayesian framework. Therefore, using a small $\eta_{\rm dec}$ can be a fix to the problem that is justified by the Bayesian principle. The third observation is that the condition for posterior collapse is completely independent of the parameter $\eta_{\rm enc}$, which is the desired variance according to the prior $p(z)$. This means that under a Gaussian assumption, the prior does not affect the posterior collapse at all.

Lastly, one also notices a potential problem. The eigenvalue of the second layer $U$ increases with $\sqrt{\beta}\eta_{\rm dec}$, while the first layer decreases with $\sqrt{\beta}\eta_{\rm dec}$, and so having a too-small $\beta$ or $\eta_{\rm dec}$ causes the model to have a very large norm, which can cause a significant problem for both empirical optimization and generalization. This problem is well-known in the studies about the use of $L_2$ regularization in deep learning: suppose we apply weight decay to two different layers of a ReLU net, and decrease the weight decay strength of one layer to zero, then the norm of this layer will tend to infinity, and the norm of the other layer will tend to zero \citep{mehta2018loss}. However, in the next section, we will see that this problem is miraculously solved for VAE when $\sigma_i$ is learnable.

\subsection{Linear VAE with Learnable $\Sigma$}\label{sec: learnable sigma}
Now, we consider the more general case of a learnable $\Sigma$. In practice, $\Sigma$ is often dependent on the input $x$. We make the simplification that $\Sigma$ is just a data-independent optimizable diagonal matrix, which is the common assumption in the related works \citep{lucas2019don}. In Section~\ref{app sec: data dependent encoder variance}, we consider the case when $\Sigma$ is data-dependent and show that our result remains unchanged. The following Corollary gives the optimal training objective as a function $\Sigma$ and is a direct consequence of proposition~\ref{proposition:min-Luv}.
\begin{corollary}\label{corollary:min-Luv-val}
\begin{align}
    \min_{U, V} L(U, V) = \sum_{i=1}^{d_1} \zeta_i^2 - \left(\zeta_i - \frac{\sqrt{\beta}\sigma_i \eta_{\rm dec}}{\eta_{\rm enc}}\right)^2 \mathbbm{1}_{\zeta_i > \frac{\sqrt{\beta}\sigma_i \eta_{\rm dec}}{\eta_{\rm enc}}} + \sum_{i=d_1+1}^{d^*} \zeta_i^2,~\label{eq:min-Luv}
\end{align}
where the indicator $\mathbbm{1}_{f > 0} = 1$ when the corresponding inequality condition $f>0$ is true, and $\mathbbm{1}_{f > 0} = 0$ otherwise.
\end{corollary}


The constant term $\sum_{i=d_1+1}^{d^*} \zeta_i^2$ in Equation~\eqref{eq:min-Luv} only appears when the latent dimension $d_1$ is less than $d^* = \min (d_0, d_2)$. This is the common situation for VAE applications. It indicates that the model learns the large eigenvalues and ignores the small eigenvalues.
This means that to find the optimal $\sigma_i$ of Eq.~\eqref{eq:general-vae-loss}, one only has to find the global minimum of a reduced objective:
\begin{align}
& \min_{U,W} L_{\rm VAE}(U, W, \Sigma) = \min_{U,V} \frac{1}{2\eta_{\rm dec}^2} L(U, V) + \sum_{i=1}^{d_1} \frac{\beta}{2} \left(\frac{\sigma_i^2}{\eta^2_{\rm enc}} - 1 - \log \frac{\sigma_i^2}{\eta_{\rm enc}^2}\right) \\
& = \frac{1}{2\eta_{\rm dec}^2} \sum_{i=1}^{d_1}\left[  \underbrace{ \zeta_i^2 - \left(\zeta_i - \frac{\sqrt{\beta}\sigma_i \eta_{\rm dec}}{\eta_{\rm enc}}\right)^2 \mathbbm{1}_{\zeta_i > \frac{\sqrt{\beta}\sigma_i \eta_{\rm dec}}{\eta_{\rm enc}}} + \beta \eta_{\rm dec}^2 \left(\frac{\sigma_i^2}{\eta^2_{\rm enc}} - 1 - \log \frac{\sigma_i^2}{\eta_{\rm enc}^2}\right)}_{:=l_i(\sigma_i)} \right] + constant.\label{eq: last objective with sigma}
\end{align}
The optimal $\sigma_i^*$ can thus be obtained by minimizing each $l_{i}$ independently: $\sigma_i^* = \argmin_{\sigma>0} l_i(\sigma)$. 

\begin{proposition}\label{proposition:min-sigma-enc}
    The optimal $\sigma_i^*$ of $l_i(\sigma)$ is
    \begin{align}
        \sigma_i^* = \left\{\begin{array}{cc}
          \frac{\sqrt{\beta}\eta_{\rm dec}}{\zeta_i} \eta_{\rm enc} & \text{if } \beta \eta_{dec}^2 < \zeta_i^2; \\
        \eta_{\rm enc} & \text{if } \beta \eta_{dec}^2 \geq \zeta_i^2.
        \end{array}\right.
    \end{align}
\end{proposition}
This proposition gives an explicit expression for $\sigma_i^*$. On the one hand, we see that there is a threshold value for $\beta$. If $\beta$ is sufficiently large, $\sigma_{i}$ will be identical to the prior value $\eta_{\rm enc}$, in agreement with our assumption in the previous section. On the other hand, the learned variance $\sigma_i^*$ is a function of $\beta$ if $\beta$ is below a threshold. We will see that this threshold is the necessary and sufficient condition for posterior collapse to happen in a learnable $\Sigma$ setting. Thus, the learned variance being identical to the prior variance is also a signature of posterior collapse. The following theorem gives the precise form of the global minimum.

\begin{theorem}~\label{thm:linear-vae}
The global minimum of $L_{\rm VAE}(U, W, \Sigma)$ is given by
\begin{align}
    U^* &= F \Lambda P,
\end{align}
$W^*$ is the solution of 
\begin{align}
    \Phi^{\frac{1}{2}} P_A^\top W = G \Theta P,
\end{align}
where $F $ and $G$ are derived by the SVD of $Z$, $P$ is an arbitrary orthogonal matrix in $\mathbb{R}^{d_1 \times d_1}$, and $\Lambda = {\rm diag}(\lambda_1, ..., \lambda_{d_1})$ and $\Theta = {\rm diag}(\theta_1, ..., \theta_{d_1})$  are diagonal matrices such that
\begin{align}
    \lambda_i &=
    \frac{1}{\eta_{\rm enc}} \sqrt{\max\left(0, \zeta_i^2 - \beta \eta_{\rm dec}^2\right) };\\
    \theta_i &\begin{cases} = \frac{\eta_{\rm enc}}{\zeta_i} \sqrt{\max\left(0, \zeta_i^2 -{\beta \eta_{\rm dec}^2}\right)} &\text{  when $\zeta_i^2 > 0$}; \\
    =0 &\text{otherwise}.
    \end{cases}
\end{align}
The optimal $\Sigma^* = {\rm diag}(\sigma_1^{*2}, ..., \sigma_{d_1}^{*2})$ such that
    \begin{align}
        \sigma_{i}^* = \left\{\begin{array}{cc}
          \frac{\sqrt{\beta}\eta_{\rm dec}}{\zeta_i} \eta_{\rm enc} & \beta \eta_{dec}^2 < \zeta_i^2, \\
        \eta_{\rm enc} & \beta \eta_{dec}^2 \geq \zeta_i^,.
        \end{array}\right.
    \end{align}
for $i \leq \min(d_0,d_2)$. For $i>\min(d_0,d_2)$, $\sigma_i^* = 0$.

\end{theorem}

\begin{proof}
The optimal solution $U^*, W^*, \Sigma^*$ are obtained by combining proposition~\ref{proposition1},~\ref{proposition:min-Luv}, and~\ref{proposition:min-sigma-enc}.
\end{proof}

Comparing with the solution in section~\ref{sec: nonlearnable}, one notices two things: (a) the conditions for complete or partial posterior collapse remain unchanged, which implies that a learnable latent variance is neither qualitatively nor quantitatively relevant for the posterior collapse problem even though the functional form of the eigenvalues changed; (b) the magnitude of each of the two layers no longer scales with $\sqrt{\beta} \eta_{\rm dec}$, and so using a small $\beta$ or $\eta$ will not directly cause the model to diverge in the norm, which suggests using that making $\Sigma$ learnable can have the unexpected practical advantage of stabilizing the training. 

Additionally, one also notices that $\beta \eta_{\rm dec}^2$ has the effect of keeping the learned model low-rank by removing all the eigenvalues of the learned model below it. This can be directly compared with the effect of using a latent dimension smaller than the input dimension: $d_1<d_0$. In the latter case, the smallest $d_1-d_0$ singular values are also pruned. There is a difference between the two types of low-rankness: using a large $\beta \eta_{\rm dec}^2$ both removes all the singular values below it and shrinks the remaining ones while using a small latent dimension only removes the smaller singular values without affecting the rest. This is similar to the difference between soft thresholding estimation and hard thresholding estimation in statistics \citep{wasserman2013all}. This suggests that partial posterior collapses are not necessarily undesirable because, during a partial posterior collapse, the latent variable models automatically perform a degree of sparse learning, which is theoretically understood to help denoising the signal and lead to better generalization \citep{markovsky2012low}. That being said, complete posterior collapse should always be avoided.

\subsection{Learnable $\eta_{\rm dec}$}
Our result can also be extended to the case when $\eta_{\rm dec}$ is learnable, which has been suggested by \cite{lucas2019don} as a remedy for the posterior collapse. To do this, we need to include the partition function of the decoder, proportional to $\log \eta_{\rm dec}^2$, that has been ignored in Eq.~\eqref{eq:general-vae-loss}. We present the detailed analysis in Section~\ref{app sec: learnable decoder variance}. Our analysis shows that even if $\eta_{\rm dec}$ is learnable, posterior collapse can happen for some datasets. In addition to the fact that it is also possible for collapses not to happen when $\eta_{\rm dec}$ is not learned, we conclude that $\eta_{\rm dec}$ does not have a causal relation with posterior collapse. Our analysis also suggests a way to fix posterior collapse for VAE: make $\eta_{\rm dec}$ learnable \underline{and} set $\beta< {d_2}/{d^*}$. Note that the condition $\beta< {d_2}/{d^*}$ is tight in the sense that if it does not hold, then there exists a data distribution such that complete collapse can happen. This condition also highlights that it is important to introduce the $\beta$ coefficient for VAEs because, for VAE, $d_2/d^*=1$, and this condition translates to $\beta <1$; namely, vanilla VAE cannot avoid complete collapse.

Our result also implies that $\beta$ has a highly nonlinear effect on learning depending on the architecture. For example, when the model is underparameterized ($d_1< d^*$), using a small $\beta$ does not cause any problem, whereas for an overparametrized model, a small $\beta$ causes the decoder variance to converge towards $0$.

\subsection{Implications}
Our main results have implications for both the problem of posterior collapse and the practice of latent variable models in general.

\textbf{The cause of posterior collapse}. One important implication is the identification of the cause of the posterior collapse problem and the potential ways to fix it. Our results suggest that
\begin{compactitem}
    \item a learnable (data-dependent or not) latent variance is not the cause of posterior collapse;
    \item changing the variance of the prior cannot fix or influence the posterior collapse problem;
    \item comparing with the results in \cite{lucas2019don}, $\eta_{\rm dec}$ being learnable or not is causally related to the posterior collapse problem;
    \item the values of $\eta_{\rm dec}$ and $\beta$ are crucial for the posterior collapse;
    \item choosing appropriate $\beta$ is still needed: a sufficiently small $\beta$ can avoid posterior collapse.
\end{compactitem}
Note that the effect of a small $\beta$ (large $\eta_{\rm dec}$) weakens the prior (reconstruction) term, and so the cause of the posterior collapse must be the competition between the prior term, which regularizes the complexity of the model, and the likelihood term, which encourages accurate recognition/reconstruction. Our results suggest that one can ignore the effect of the $\ell_{var}$ term in studying the mechanism of posterior collapse. Ignoring the $\ell_{var}$ Eq.~\ref{eq:general-vae-loss}, one sees that the posterior collapse is caused by the competition between the likelihood and $\ell_{mean}$, which is precisely the regularization effect on the mean of $z$.


There is an interesting alternative perspective on the nature of the posterior from the viewpoint of the loss landscape geometry. The following theorem states that the origin (where all parameters are zero) is either a saddle or the global minimum for this problem. Since we have shown that $\sigma_i$ does not affect the collapse, we simply let $\sigma_i=
\eta_{\rm enc}$ as in Section~\ref{sec: nonlearnable}.
\begin{theorem}\label{theo: saddle point condition}
    The Hessian of Eq.~\ref{eq:general-vae-loss} at $0$ is positive semidefinite if and only if it is the global minimum.
\end{theorem}
The surprising aspect is that for the latent variable model, there is no intermediate case where the origin is a local minimum but not global. Therefore, the origin is, in fact, a very special point in the landscape of a latent variable model, in the sense that a key global property of the landscape (namely, the global minimum) is determined by the local geometry of the model at the origin. Noting that our model can be seen as a direct generalization of the Bayesian linear regression to a deeper architecture, it also becomes reasonable to suspect that the posterior collapse problem is a unique problem of deep learning because the standard Bayesian linear regression does not suffer posterior collapse because the origin can never be a local maximum of the posterior \citep{bishop2006pattern}. \cite{dai2020usual} also finds the origin to be a very special point in a general deep nonlinear VAE structure and that it can be a local minimum under various settings. However, the implication of our work is broader. The origin is not only a special point for the autoencoding model families but can actually be a special point for a very broad of model classes (namely, the model class of general latent variable models). The problem of posterior collapse is thus not limited to autoencoders but can also be relevant to common regression and classification tasks.

\begin{figure}[t!]
    \centering
    \vspace{-1em}
    \includegraphics[width=.4\linewidth]{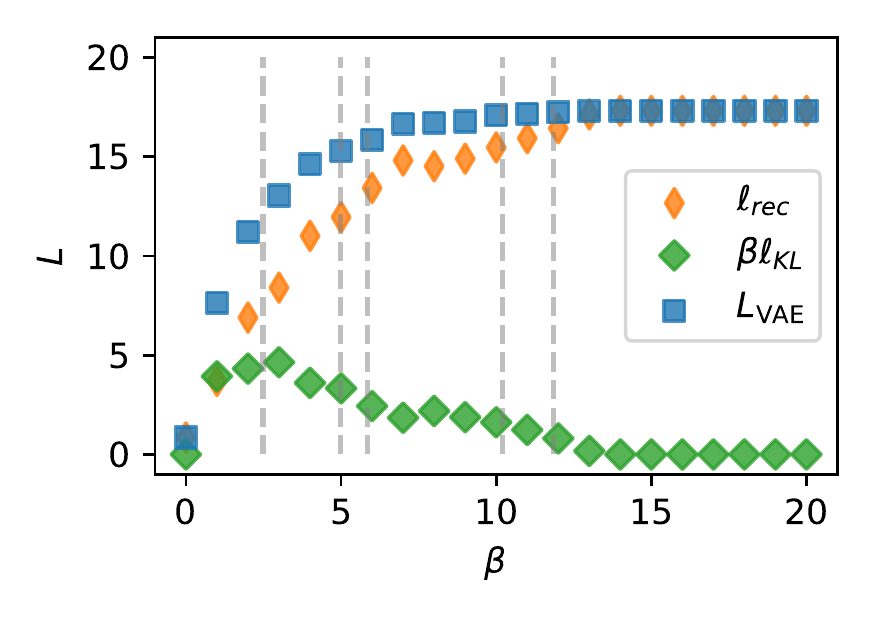}
    \includegraphics[width=.4\linewidth]{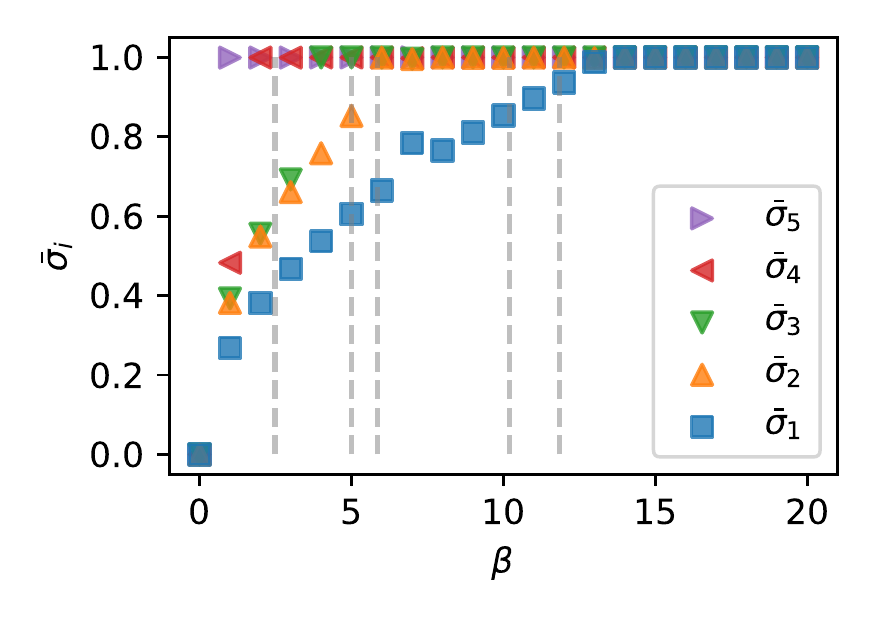}
        \vspace{-1em}
    \caption{Training loss $L$ and $\bar \sigma_i$ versus $\beta$ on synthetic regression dataset. $\bar\sigma_i$ is measured by averaging over the training set. The vertical dashed lines show where the theory predicts a partial collapse. Complete posterior collapse happens at roughly $\beta = 14$.}
    \label{fig:non-linear-relu-regression}
    \centering
    \includegraphics[width=.4\linewidth]{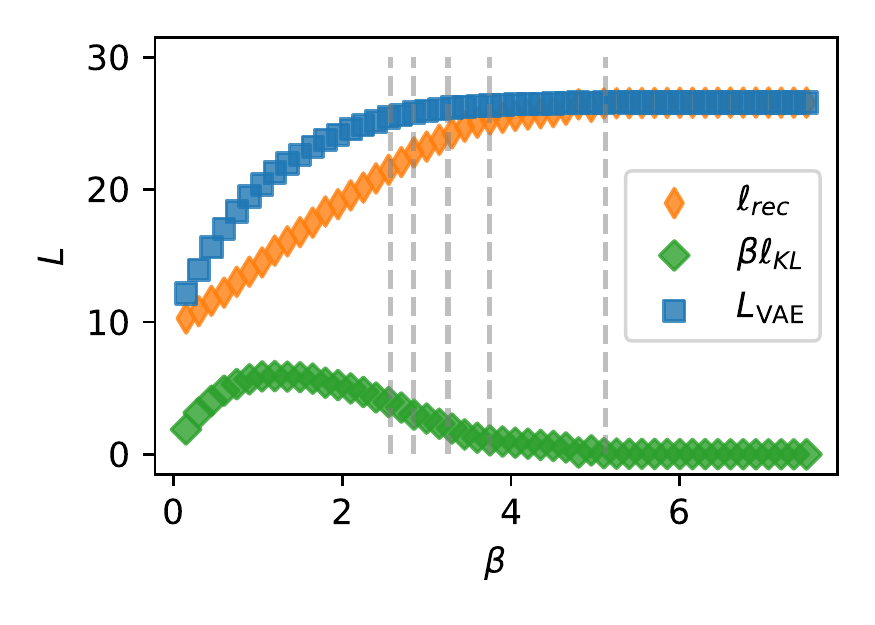}
    \includegraphics[width=.4\linewidth]{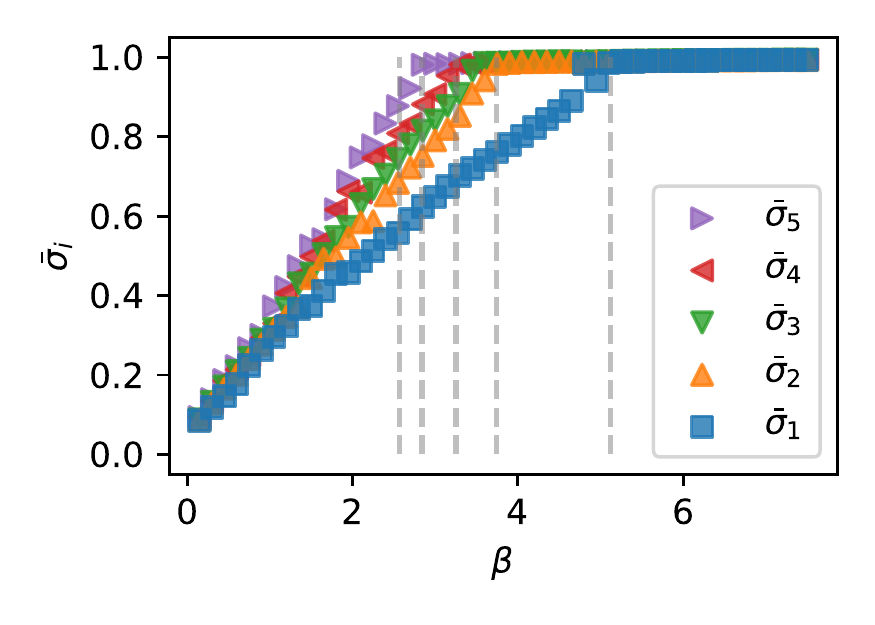}
        \vspace{-1em}

    \caption{Training loss and $\bar\sigma_i$ versus $\beta$ for MNIST dataset. The vertical dashed lines show where the theory predicts a partial collapse. The posterior collapse happens for the MNIST dataset at around $\beta=5$.}
    \vspace{-1em}
    \label{fig:non-linear-MNIST}
\end{figure}

\textbf{Connection to other types of collapses}. Our result suggests that there are some interesting connections between the posterior collapse phenomenon and the neural collapse phenomenon in supervised learning \citep{papyan2020prevalence} and dimensional collapse phenomenon in self-supervised learning (SSL) \citep{jing2021understanding}. \cite{ziyin2022exact} and \cite{ziyin2022exactb} shows that the neural collapse phenomenon for a two-layer model can be understood through the change of the stability at the origin, which is determined by the competition between the signal strength ($\E[xy]$) of the data distribution and the regularization strength of weight decay. For SSL, \cite{ziyin2022shapes} also shows that the stability of the origin is important and that it is decided by the competition between the level of data variation and the data augmentation strength. Our result suggests that the posterior collapse problem can also be understood through the stability at the origin. This might imply that there could be some universal cause of all these collapses that have been discovered independently in different subfields of deep learning, and one important future direction would be to study these phenomena from a unified perspective.

\textbf{Insights for latent variable model practices}. While we have primarily focused on discussing the phenomena of posterior collapse, our results also shed light on latent variable models (including VAE) in practice when there is no complete posterior collapse. Specifically, our results suggest that
\begin{compactitem}
    \item latent variable models perform sparse learning through soft thresholding or hard thresholding or both;
    \item thus, \textit{partial} posterior collapse may actually be desirable;
    \item making the latent variance learnable can help stabilize training and avoid divergence of model parameters;
    \item when $\eta_{\rm dec}$ is not learned, the effect of increasing $\beta$ is identical to the effect of decreasing $\eta_{\rm dec}$;
    \item when $\eta_{\rm dec}$ is learned, one needs to pay special care to choose a suitable $\beta$.

\end{compactitem}

\begin{figure}[t!]
    \centering
    \begin{subfigure}[b]{0.32\textwidth}
        \centering
        \includegraphics[trim={2cm 3cm 1.8cm 3cm},clip,width=\textwidth]{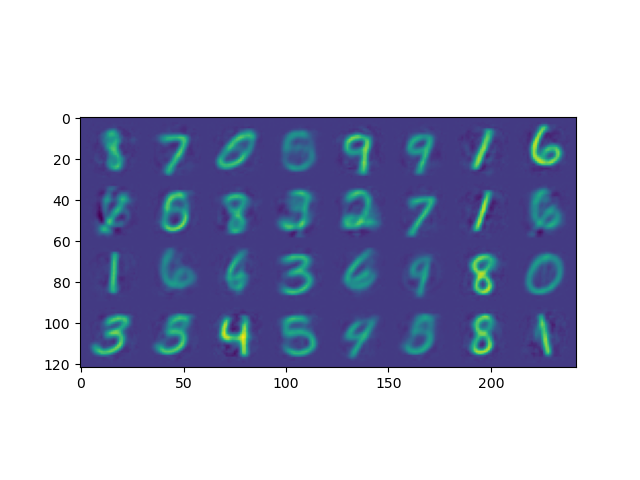}
        \caption{$\beta=1$, remaining modes = 5}
    \end{subfigure}
    \hfill
    \begin{subfigure}[b]{0.32\textwidth}
        \centering
        \includegraphics[trim={2cm 3cm 1.8cm 3cm},clip,width=\textwidth]{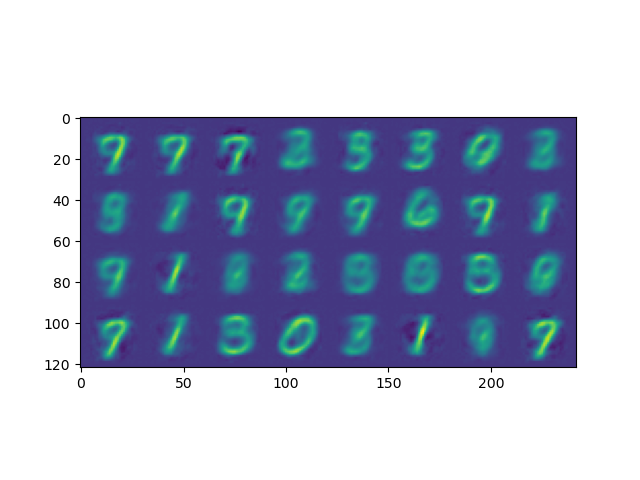}
        \caption{$\beta=2.75$, remaining modes: 4}
    \end{subfigure}
    \hfill
    \begin{subfigure}[b]{0.32\textwidth}
        \centering
        \includegraphics[trim={2cm 3cm 1.8cm 3cm},clip,width=\textwidth]{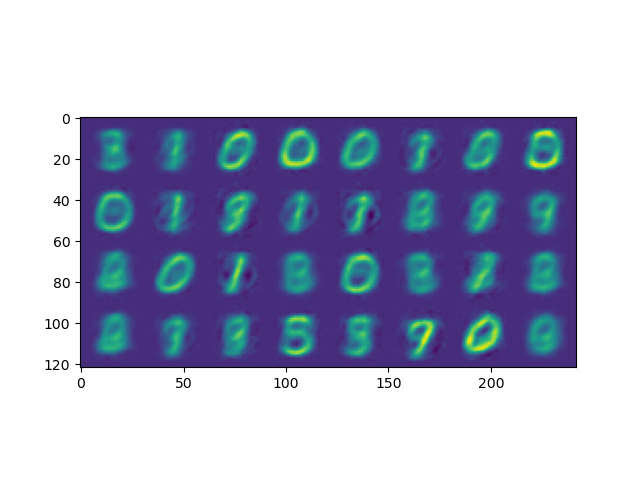}
        \caption{$\beta=3$, remaining modes: 3}
    \end{subfigure}
        \begin{subfigure}[b]{0.32\textwidth}
        \centering
        \includegraphics[trim={2cm 3cm 1.8cm 3cm},clip,width=\textwidth]{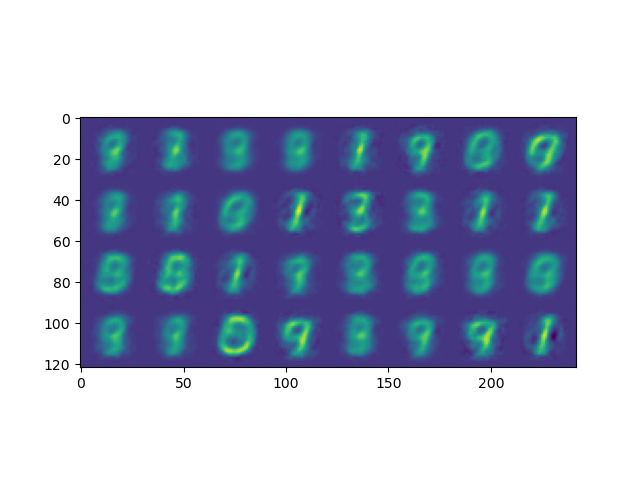}
        \caption{$\beta=3.5$, remaining modes: 2}
    \end{subfigure}
    \hfill
    \begin{subfigure}[b]{0.32\textwidth}
        \centering
        \includegraphics[trim={2cm 3cm 1.8cm 3cm},clip,width=\textwidth]{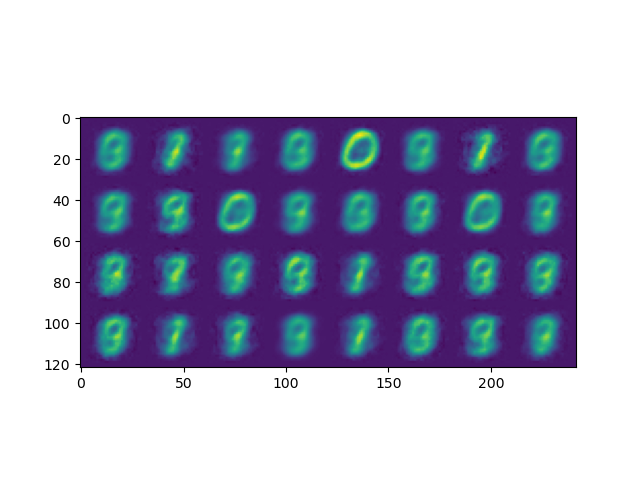}
        \caption{$\beta=4$, remaining modes: 1}
    \end{subfigure}
    \hfill
    \begin{subfigure}[b]{0.32\textwidth}
        \centering
        \includegraphics[trim={2cm 3cm 1.8cm 3cm},clip,width=\textwidth]{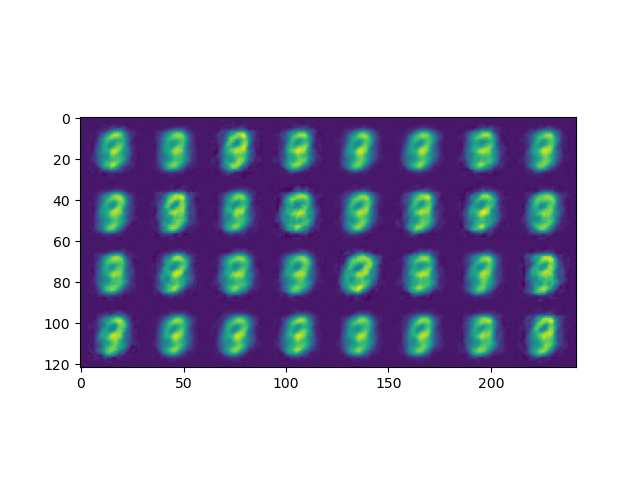}
        \caption{$\beta=6$, remaining modes: 0}
    \end{subfigure}
    \caption{MNIST generation under different $\beta$. The generated images lose diversity and variation as $\beta$ increases. The number of modes left is estimated by the theoretical prediction of thresholds of each singular value.}
    \label{fig:MNIST-gen}
    \vspace{-1em}
\end{figure}

\vspace{-2mm}
\section{Numerical Examples}\label{sec: experiment}
\vspace{-1mm}

This section empirically examines our theoretical claims for linear models and demonstrates that our key theoretical insights generalize well to nonlinear models and natural data.  

\textit{Setting.} We illustrate our results on both synthetic data and natural data. For synthetic data, we sample input data $x$ from multivariate normal distribution $\mathcal{N}(0, A)$, and target data $y = Mx$ is obtained by a linear transformation. Specifically, we choose $d_0 = d_2 = 5$. As an example of natural data, we also experiment with the standard MNIST data. 
Following common practices, we choose $\eta_{\rm dec} = \eta_{\rm enc} = 1$. For non-linear VAE models, we consider two-layer fully connected neural networks for the encoder and decoder with both ReLU and Tanh activation functions and with hidden dimension $d_h$. For synthetic dataset $d_h=8$, and $d_h = 2048$ for real-world data. In contrast to our assumption that the variances $\Sigma$ of encoded $z$ are independent from the input $x$, we parameterize the variance of each encoded $z$ by a linear transformation or a two-layer neural network, i.e., $\Sigma(x) = {\rm [Linear/MLP]}(x)$. This data-dependent modeling is closer to the common practice, and the comparison can justify the correctness of our theory. The model is optimized by Adam with a learning rate of $10^{-3}$. The results are reported after the convergence. For MNIST, the learning rate is $10^{-4}$. 

\textit{Results}. Linear models are found to agree precisely with the theoretical results, so we only present the results in the appendix. We focus on exploring the nonlinear models in the main text. We first consider a simple regression task with MLP encoder and decoders with the ReLU activation (Figure~\ref{fig:non-linear-relu-regression}). Here, we see that the theoretical prediction of loss function $L_{\rm VAE}$ agrees well with empirical observation. Moreover, the threshold of complete posterior collapse is also perfectly predicted. For completeness, we also present the case when (1) the activation is Tanh in Appendix~\ref{app:regression}. We note that the results are similar. The observation is similar to the standard MNIST dataset with a nonlinear encoder and decoder. See Figure~\ref{fig:non-linear-MNIST}.
For illustration, we also present the generated MNIST images by non-linear $\beta$-VAE trained with different choices of $\beta$ in Figure~\ref{fig:MNIST-gen}. The latent dimension is five as described before. When there are $5$ non-collapsed modes, the generated images are both sharp and contain meaningful variations. As the number of remaining non-collapsed modes reduces to zero, we see that the generated images become increasingly blurred, and the variation between the data also diminishes. When the model completely collapses, the model outputs a constant, as the theory suggests. Moreover, we note that the values of $\beta$ are chosen according to the theoretical thresholds for each mode to collapse, i.e, the top-5 $\zeta_i$ are $[5.12, 3.74, 3.25, 2.84, 2.57]$. We see that the theoretical thresholds provide good predictive power for the behavior of mode collapse qualitatively.

\vspace{-2mm}
\section{Outlook}
\vspace{-1mm}

In this work, we have tackled the problem of posterior collapse from a loss landscape point of view. Our work also contributes to the fundamental theory of deep learning. The linear VAE architecture can be seen as a deep linear model with two layers, whose loss landscape is highly nontrivial. In this perspective, our results advance those results in \cite{ziyin2022exact}, where the dimension of the output space is limited to 1d. The limitation of our work is obvious: our theory only deals with the landscape, and it is unclear how the dynamics of gradient-based methods could contribute to the collapse problem. In fact, there is strong evidence that stochastic gradient descent can bias the model towards low-rank or sparse solutions \citep{arora2019implicit, ziyin2021sgd}, and, in the context of posterior collapse, these are precisely the collapsed solutions. One important future direction is thus to study the role of dynamics in influencing posterior collapse.

\bibliographystyle{apalike}
\bibliography{ref}

\begin{thebibliography}{}

\bibitem[Alemi et~al., 2018]{alemi2018fixing}
Alemi, A., Poole, B., Fischer, I., Dillon, J., Saurous, R.~A., and Murphy, K.
  (2018).
\newblock Fixing a broken elbo.
\newblock In {\em International Conference on Machine Learning}, pages
  159--168. PMLR.

\bibitem[Arora et~al., 2019]{arora2019implicit}
Arora, S., Cohen, N., Hu, W., and Luo, Y. (2019).
\newblock Implicit regularization in deep matrix factorization.
\newblock {\em Advances in Neural Information Processing Systems}, 32.

\bibitem[Bishop and Nasrabadi, 2006]{bishop2006pattern}
Bishop, C.~M. and Nasrabadi, N.~M. (2006).
\newblock {\em Pattern recognition and machine learning}, volume~4.
\newblock Springer.

\bibitem[Bowman et~al., 2015]{bowman2015generating}
Bowman, S.~R., Vilnis, L., Vinyals, O., Dai, A.~M., Jozefowicz, R., and Bengio,
  S. (2015).
\newblock Generating sentences from a continuous space.
\newblock {\em arXiv preprint arXiv:1511.06349}.

\bibitem[Dai et~al., 2020]{dai2020usual}
Dai, B., Wang, Z., and Wipf, D. (2020).
\newblock The usual suspects? reassessing blame for vae posterior collapse.
\newblock In {\em International Conference on Machine Learning}, pages
  2313--2322. PMLR.

\bibitem[Friston, 2009]{friston2009free}
Friston, K. (2009).
\newblock The free-energy principle: a rough guide to the brain?
\newblock {\em Trends in cognitive sciences}, 13(7):293--301.

\bibitem[Hardt and Ma, 2016]{hardt2016identity}
Hardt, M. and Ma, T. (2016).
\newblock Identity matters in deep learning.
\newblock {\em arXiv preprint arXiv:1611.04231}.

\bibitem[Hastie et~al., 2019]{hastie2019surprises}
Hastie, T., Montanari, A., Rosset, S., and Tibshirani, R.~J. (2019).
\newblock Surprises in high-dimensional ridgeless least squares interpolation.
\newblock {\em arXiv preprint arXiv:1903.08560}.

\bibitem[Higgins et~al., 2016]{higgins2016beta}
Higgins, I., Matthey, L., Pal, A., Burgess, C., Glorot, X., Botvinick, M.,
  Mohamed, S., and Lerchner, A. (2016).
\newblock beta-vae: Learning basic visual concepts with a constrained
  variational framework.

\bibitem[Huang et~al., 2018]{huang2018improving}
Huang, C.-W., Tan, S., Lacoste, A., and Courville, A.~C. (2018).
\newblock Improving explorability in variational inference with annealed
  variational objectives.
\newblock {\em Advances in Neural Information Processing Systems}, 31.

\bibitem[Jiang and Ahn, 2020]{jiang2020generative}
Jiang, J. and Ahn, S. (2020).
\newblock Generative neurosymbolic machines.
\newblock {\em Advances in Neural Information Processing Systems},
  33:12572--12582.

\bibitem[Jing et~al., 2021]{jing2021understanding}
Jing, L., Vincent, P., LeCun, Y., and Tian, Y. (2021).
\newblock Understanding dimensional collapse in contrastive self-supervised
  learning.
\newblock {\em arXiv preprint arXiv:2110.09348}.

\bibitem[Kawaguchi, 2016]{kawaguchi2016deep}
Kawaguchi, K. (2016).
\newblock Deep learning without poor local minima.
\newblock {\em Advances in Neural Information Processing Systems}, 29:586--594.

\bibitem[Kingma et~al., 2016]{kingma2016improved}
Kingma, D.~P., Salimans, T., Jozefowicz, R., Chen, X., Sutskever, I., and
  Welling, M. (2016).
\newblock Improved variational inference with inverse autoregressive flow.
\newblock {\em Advances in neural information processing systems}, 29.

\bibitem[Kingma and Welling, 2013]{kingma2013auto}
Kingma, D.~P. and Welling, M. (2013).
\newblock Auto-encoding variational bayes.
\newblock {\em arXiv preprint arXiv:1312.6114}.

\bibitem[Laurent and Brecht, 2018]{laurent2018deep}
Laurent, T. and Brecht, J. (2018).
\newblock Deep linear networks with arbitrary loss: All local minima are
  global.
\newblock In {\em International conference on machine learning}, pages
  2902--2907. PMLR.

\bibitem[Liu, 2021]{liu2021relational}
Liu, K.-H. (2021).
\newblock Relational learning with variational bayes.
\newblock In {\em International Conference on Learning Representations}.

\bibitem[Lucas et~al., 2019]{lucas2019don}
Lucas, J., Tucker, G., Grosse, R.~B., and Norouzi, M. (2019).
\newblock Don't blame the elbo! a linear vae perspective on posterior collapse.
\newblock {\em Advances in Neural Information Processing Systems}, 32.

\bibitem[L{\"u}cke et~al., 2020]{lucke2020evidence}
L{\"u}cke, J., Forster, D., and Dai, Z. (2020).
\newblock The evidence lower bound of variational autoencoders converges to a
  sum of three entropies.
\newblock {\em arXiv preprint arXiv:2010.14860}.

\bibitem[Mackay, 1992]{mackay1992bayesian}
Mackay, D. J.~C. (1992).
\newblock {\em Bayesian methods for adaptive models}.
\newblock PhD thesis, California Institute of Technology.

\bibitem[Markovsky, 2012]{markovsky2012low}
Markovsky, I. (2012).
\newblock {\em Low rank approximation: algorithms, implementation,
  applications}, volume 906.
\newblock Springer.

\bibitem[Mehta et~al., 2018]{mehta2018loss}
Mehta, D., Chen, T., Tang, T., and Hauenstein, J.~D. (2018).
\newblock The loss surface of deep linear networks viewed through the algebraic
  geometry lens.
\newblock {\em arXiv preprint arXiv:1810.07716}.

\bibitem[Nakajima and Sugiyama, 2010]{nakajima2010implicit}
Nakajima, S. and Sugiyama, M. (2010).
\newblock Implicit regularization in variational bayesian matrix factorization.
\newblock In {\em ICML}.

\bibitem[Nakajima et~al., 2013]{nakajima2013global}
Nakajima, S., Sugiyama, M., Babacan, S.~D., and Tomioka, R. (2013).
\newblock Global analytic solution of fully-observed variational bayesian
  matrix factorization.
\newblock {\em The Journal of Machine Learning Research}, 14(1):1--37.

\bibitem[Nakajima et~al., 2015]{nakajima2015condition}
Nakajima, S., Tomioka, R., Sugiyama, M., and Babacan, S.~D. (2015).
\newblock Condition for perfect dimensionality recovery by variational bayesian
  pca.
\newblock {\em J. Mach. Learn. Res.}, 16:3757--3811.

\bibitem[Neal, 2012]{neal2012bayesian}
Neal, R.~M. (2012).
\newblock {\em Bayesian learning for neural networks}, volume 118.
\newblock Springer Science \& Business Media.

\bibitem[Papyan et~al., 2020]{papyan2020prevalence}
Papyan, V., Han, X., and Donoho, D.~L. (2020).
\newblock Prevalence of neural collapse during the terminal phase of deep
  learning training.
\newblock {\em Proceedings of the National Academy of Sciences},
  117(40):24652--24663.

\bibitem[Razavi et~al., 2019]{razavi2019preventing}
Razavi, A., Oord, A. v.~d., Poole, B., and Vinyals, O. (2019).
\newblock Preventing posterior collapse with delta-vaes.
\newblock {\em arXiv preprint arXiv:1901.03416}.

\bibitem[Saxe et~al., 2013]{saxe2013exact}
Saxe, A.~M., McClelland, J.~L., and Ganguli, S. (2013).
\newblock Exact solutions to the nonlinear dynamics of learning in deep linear
  neural networks.
\newblock {\em arXiv preprint arXiv:1312.6120}.

\bibitem[Shekhovtsov et~al., 2022]{shekhovtsov2022vae}
Shekhovtsov, A., Schlesinger, D., and Flach, B. (2022).
\newblock {VAE} approximation error: {ELBO} and exponential families.
\newblock In {\em International Conference on Learning Representations}.

\bibitem[Sohn et~al., 2015]{sohn2015learning}
Sohn, K., Lee, H., and Yan, X. (2015).
\newblock Learning structured output representation using deep conditional
  generative models.
\newblock {\em Advances in neural information processing systems}, 28.

\bibitem[Von~Neumann, 1962]{von1962some}
Von~Neumann, J. (1962).
\newblock Some matrix inequalities and metrization of matrix space, tomask
  university review 1 (1937) 286-300. reprinted in ah taub (ed.), john von
  neumann collected works (vol. iv).

\bibitem[Wang and Yeung, 2020]{wang2020survey}
Wang, H. and Yeung, D.-Y. (2020).
\newblock A survey on bayesian deep learning.
\newblock {\em ACM Computing Surveys (CSUR)}, 53(5):1--37.

\bibitem[Wang et~al., 2021]{wang2021posterior}
Wang, Y., Blei, D., and Cunningham, J.~P. (2021).
\newblock Posterior collapse and latent variable non-identifiability.
\newblock {\em Advances in Neural Information Processing Systems}, 34.

\bibitem[Wasserman, 2013]{wasserman2013all}
Wasserman, L. (2013).
\newblock {\em All of statistics: a concise course in statistical inference}.
\newblock Springer Science \& Business Media.

\bibitem[Zhao et~al., 2021]{zhao2021assessment}
Zhao, M., Hoti, K., Wang, H., Raghu, A., and Katabi, D. (2021).
\newblock Assessment of medication self-administration using artificial
  intelligence.
\newblock {\em Nature medicine}, 27(4):727--735.

\bibitem[Zheng et~al., 2018]{zheng2018regularized}
Zheng, S., Ding, C., and Nie, F. (2018).
\newblock Regularized singular value decomposition and application to
  recommender system.
\newblock {\em arXiv preprint arXiv:1804.05090}.

\bibitem[Ziyin et~al., 2022a]{ziyin2022exact}
Ziyin, L., Li, B., and Meng, X. (2022a).
\newblock Exact solutions of a deep linear network.

\bibitem[Ziyin et~al., 2021]{ziyin2021sgd}
Ziyin, L., Li, B., Simon, J.~B., and Ueda, M. (2021).
\newblock Sgd can converge to local maxima.
\newblock In {\em International Conference on Learning Representations}.

\bibitem[Ziyin et~al., 2022b]{ziyin2022shapes}
Ziyin, L., Lubana, E.~S., Ueda, M., and Tanaka, H. (2022b).
\newblock What shapes the loss landscape of self-supervised learning?
\newblock {\em arXiv preprint arXiv:2210.00638}.

\bibitem[Ziyin and Ueda, 2022]{ziyin2022exactb}
Ziyin, L. and Ueda, M. (2022).
\newblock Exact phase transitions in deep learning.
\newblock {\em arXiv preprint arXiv:2205.12510}.

\bibitem[Ziyin et~al., 2022c]{ziyin2022stochastic}
Ziyin, L., Zhang, H., Meng, X., Lu, Y., Xing, E., and Ueda, M. (2022c).
\newblock Stochastic neural networks with infinite width are deterministic.

\end{thebibliography}
\section*{Checklist}

\begin{enumerate}

\item For all authors...
\begin{enumerate}
  \item Do the main claims made in the abstract and introduction accurately reflect the paper's contributions and scope?
    \answerYes{}
  \item Did you describe the limitations of your work?
    \answerYes{}
  \item Did you discuss any potential negative societal impacts of your work?
    \answerNA{}
  \item Have you read the ethics review guidelines and ensured that your paper conforms to them?
    \answerYes{}
\end{enumerate}

\item If you are including theoretical results...
\begin{enumerate}
  \item Did you state the full set of assumptions of all theoretical results?
    \answerYes{}
        \item Did you include complete proofs of all theoretical results?
    \answerYes{See Appendix.}
\end{enumerate}

\item If you ran experiments...
\begin{enumerate}
  \item Did you include the code, data, and instructions needed to reproduce the main experimental results (either in the supplemental material or as a URL)?
    \answerNo{The experiments are only for demonstration and are straightforward to reproduce following the theory.}
  \item Did you specify all the training details (e.g., data splits, hyperparameters, how they were chosen)?
    \answerYes{}
        \item Did you report error bars (e.g., with respect to the random seed after running experiments multiple times)?
    \answerNo{The fluctuations are visually negligible.}
        \item Did you include the total amount of compute and the type of resources used (e.g., type of GPUs, internal cluster, or cloud provider)?
    \answerYes{They are done on a single 3080Ti GPU.}
\end{enumerate}

\item If you are using existing assets (e.g., code, data, models) or curating/releasing new assets...
\begin{enumerate}
  \item If your work uses existing assets, did you cite the creators?
    \answerNA{}
  \item Did you mention the license of the assets?
    \answerNA{}
  \item Did you include any new assets either in the supplemental material or as a URL?
    \answerNA{}
  \item Did you discuss whether and how consent was obtained from people whose data you're using/curating?
    \answerNA{}
  \item Did you discuss whether the data you are using/curating contains personally identifiable information or offensive content?
    \answerNA{}
\end{enumerate}

\item If you used crowdsourcing or conducted research with human subjects...
\begin{enumerate}
  \item Did you include the full text of instructions given to participants and screenshots, if applicable?
    \answerNA{}
  \item Did you describe any potential participant risks, with links to Institutional Review Board (IRB) approvals, if applicable?
    \answerNA{}
  \item Did you include the estimated hourly wage paid to participants and the total amount spent on participant compensation?
    \answerNA{}
\end{enumerate}

\end{enumerate}

\clearpage
\appendix

\section{Additional Experiments}



\subsection{Regression}\label{app:regression}
Figure~\ref{fig:linear-regression} shows the results of a linear model in the regression setting.
Figure~\ref{fig:tanh-regression} shows the performance of Tanh MLP in the regression setting.
The complete posterior collapse is well predicted by our theory.

\begin{figure}[h]
    \centering
    \includegraphics[width=.4\linewidth]{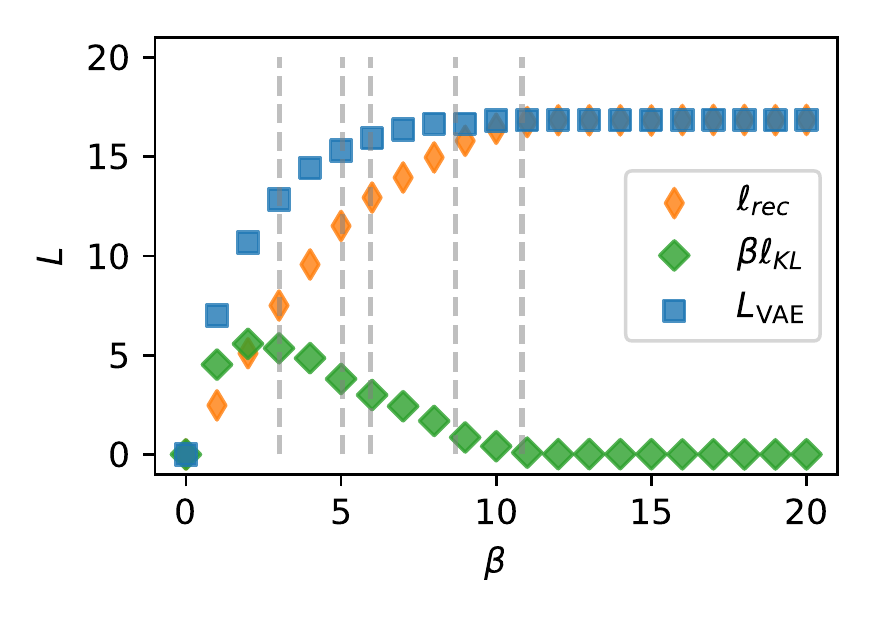}
    \includegraphics[width=.4\linewidth]{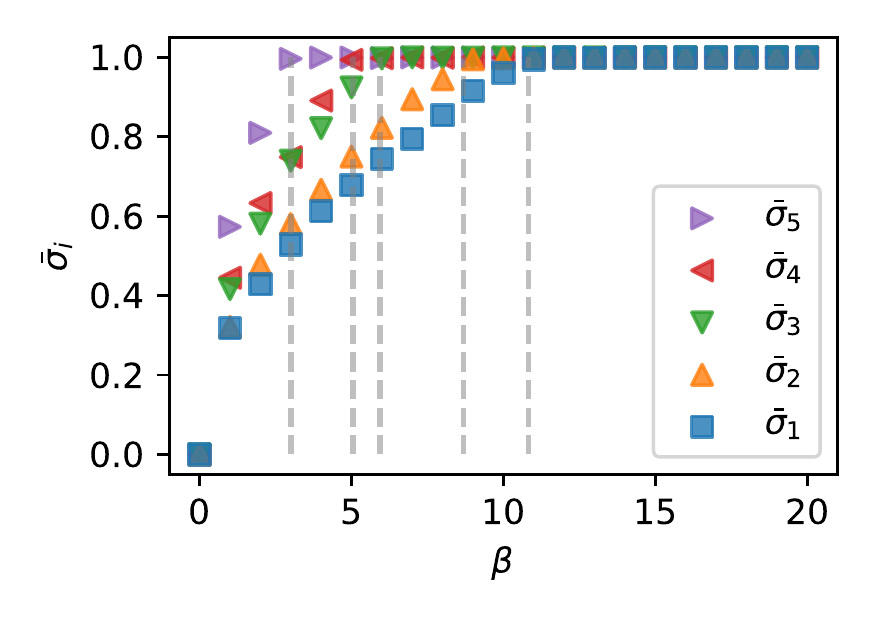}
    \caption{Training loss $L$ and $\sigma_i$ versus $\beta$ for linear regression. The theoretical prediction is plotted as vertical dashed lines.}
    \label{fig:linear-regression}
\end{figure}

\begin{figure}[h]
    \centering
    \includegraphics[width=.4\linewidth]{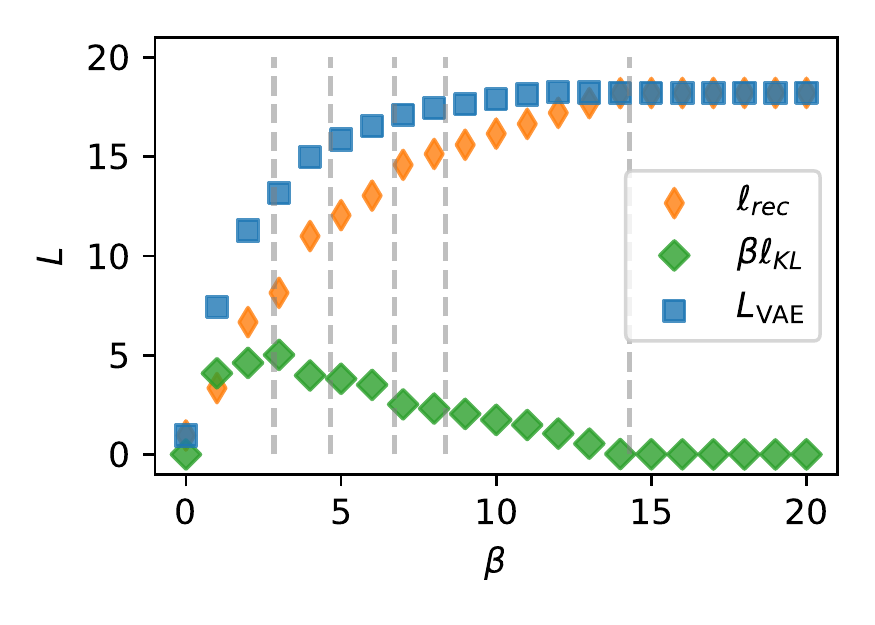}
    \includegraphics[width=.4\linewidth]{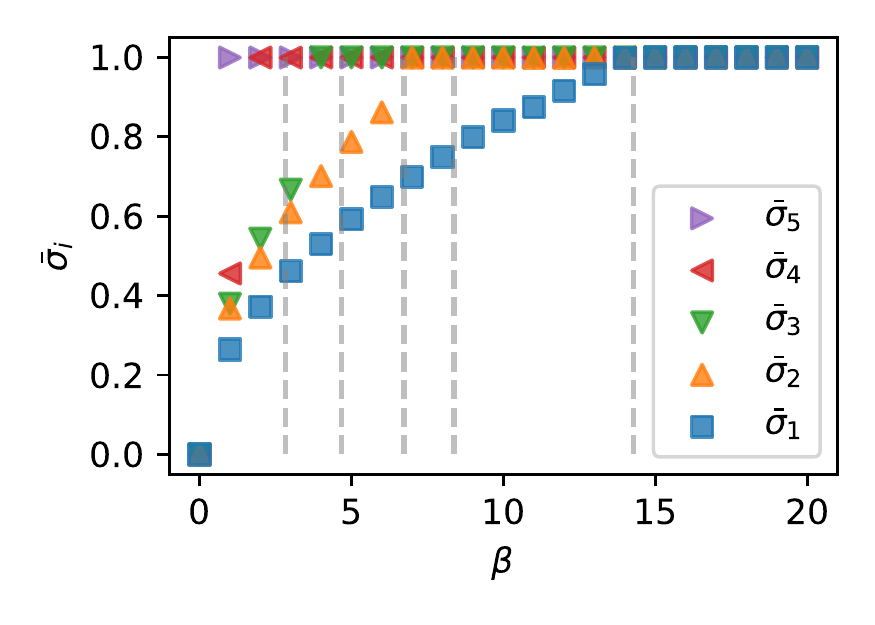}
    \caption{Training loss $L$ and $\sigma_i$ versus $\beta$ for MLP encoder and decoder with Tanh activation function. The theoretical prediction is plotted as vertical dashed lines.}
    \label{fig:tanh-regression}
\end{figure}

\clearpage

\section{Effect of Bias}~\label{app:bias-term}
Here, we study a general linear encoding and decoding model equipped with a bias term. Following the previous notation, the encoder is $z=W^\top x + b_e + \epsilon$ and the decoded distribution is $p(y|z) = \mathcal{N}(Uz + b_d, \eta_{\rm dec}^2I)$. Then, the objective of general VAE reads
\begin{align}
    L_{\rm VAE}(W, U, b_e, b_d, \Sigma) = & \frac{1}{2\eta_{\rm dec}^2} \mathbb{E}_{x,\epsilon} \left[ \|U(W^\top x + b_e + \epsilon) + b_d - y\|^2 + \beta \frac{\eta_{\rm dec}^2}{\eta_{\rm enc}^2} \|W^\top x + b_e\|^2 \right] \\
    & + \sum_{i=1}^{d_1} \frac{\beta}{2} \left(\frac{\sigma_{i}^2}{\eta^2_{\rm enc}} - 1 - \log \frac{\sigma_{i}^2}{\eta_{\rm enc}^2}\right).
\end{align}

One can show that at optima, the learned biases must take the following form.
\begin{proposition}
    The optimal biases are $b_e^* = - W^\top \mathbb{E}_{x} [x]$ and $b_d^* = \mathbb{E}_{x} [y]$.
\end{proposition}
\begin{proof}
The gradient of $L_{\rm VAE}$ with respect to $b_e$ and $b_d$ are zero when $b_e$ and $b_d$ are optimal. That is,
\begin{align}
    \frac{\partial L_{\rm VAE}}{\partial b_e} & = \frac{1}{\eta_{\rm dec}^2} \mathbb{E}_{x,\epsilon}\left[U^\top (U(W^\top x + b_e + \epsilon) + b_d - y) + \beta \frac{\eta_{\rm dec}^2}{\eta_{\rm enc}^2} (W^\top x + b_e) \right] \\
    & = \frac{1}{\eta_{\rm dec}^2} \left[\left(U^\top U + \beta \frac{\eta_{\rm dec}^2}{\eta_{\rm enc}^2}  I \right)(W^\top \mathbb{E}_{x} x + b_e) + U^\top(b_d -  \mathbb{E}_{x} y) \right] = 0,
\end{align}
and,
\begin{align}
    \frac{\partial L_{\rm VAE}}{\partial b_d} & = \frac{1}{\eta_{\rm dec}^2} \mathbb{E}_{x,\epsilon} (U(W^\top x + b_e + \epsilon) + b_d - y) \\
    & = \frac{1}{\eta_{\rm dec}^2} \left[U (W^\top \mathbb{E}_{x} x + b_e) + (b_d - \mathbb{E}_x y)\right] = 0.
\end{align}
Those condition holds if and only if $b_e^* = - W^\top \mathbb{E}_{x} x$ and $b_d^* = \mathbb{E}_{x} y$.
\end{proof}

In particular, this means that the effect of a learnable encoder bias is the same as a data-preprocessing scheme of making $x$ zero-mean. The effect of a learnable decoder bias is the same as a data-preprocessing scheme of making $y$ zero-mean.

\section{Case of a Data-Dependent Encoding Variance}\label{app sec: data dependent encoder variance}

For completeness, we extend the result in Section~\ref{sec: learnable sigma} and consider the case when the learnable variance of the latent variable $z$ is $x$-dependent, which is common in practice. Meanwhile, one might also consider the case when the variance in the decoder is learned: for concision, we do not consider this case because it is rather rare in practice.

In the same spirit, we consider the simplest case of a data-dependent variance, where the standard derivation of $z$ linearly depends on $x$. We will see that in this case, the system is no longer analytically solvable. The standard deviation is
\begin{align}
    \sigma(x) = {\rm diag}(|C x + f|) := {\rm diag} (\sigma_1(x), ..., \sigma_{d_1}(x)),
\end{align}
where $C\in \mathbb{R}^{d_1\times d_0}$ and $f\in\mathbb{R}^{d_1}$ are the learnable parameters. The latent variable $z$ is thus generated by
\begin{align}
    z = Wx + \sigma(x) \epsilon = Wx + {\rm diag}(Cx + f)\epsilon,
\end{align}
where $\epsilon \sim \mathcal{N}(0, I_{d_1})$. To emphasize the important terms, we further assume that $x$ is zero-mean: $\E[x]=0$.\footnote{As we have shown, this can be precisely achieved when the encoder and encoder have a learnable bias.} 

Using this definition of $\sigma$ in Eq.~\eqref{eq:raw_VAE_loss}, one obtains the following objective with Data-Dependent (encoding) Variance (DDV):
\begin{align}
    L_{\rm VAE}^{\rm DDV}(U&, W, C, f)\\
    = & \frac{1}{2\eta_{\rm dec}^2} \mathbb{E}_{x,\epsilon} \bigg[ \|(U W^\top x  - y) + U\sigma(x)\epsilon\|^2 + \beta \frac{\eta_{\rm dec}^2}{\eta_{\rm enc}^2} \|W^\top x\|^2 \bigg]\nonumber\\
    &\quad\quad +  \frac{\beta}{2} \sum_{i=1}^{d_1} \mathbb{E}_{x} \left(\frac{\sigma_{i}^2(x)}{\eta^2_{\rm enc}} - 1 - \log \frac{\sigma_{i}^2(x)}{\eta_{\rm enc}^2}\right)\\
    = & \frac{1}{2\eta_{\rm dec}^2} \mathbb{E}_{x,\epsilon} \bigg[ \|U W^\top x  - y\|^2 + {\rm Tr}(U\sigma(x)\epsilon \epsilon^\top \sigma(x) U^\top) + 2 {\rm Tr}(U (W^\top x  - y)\epsilon^\top \sigma(x) U^\top ) \nonumber \\
    &\quad\quad+ \beta \frac{\eta_{\rm dec}^2}{\eta_{\rm enc}^2} \|W^\top x\|^2 \bigg] +  \frac{\beta}{2} \sum_{i=1}^{d_1} \mathbb{E}_{x} \left(\frac{\sigma_{i}^2(x)}{\eta^2_{\rm enc}} - 1 - \log \frac{\sigma_{i}^2(x)}{\eta_{\rm enc}^2}\right).
\end{align}

The relevant expectation values can be computed easily:
\begin{align}\label{eq: equivalent loss function data dependent z}
    \mathbb{E}_{x} \sigma_i^2(x) & = f_i^2 + C_{i\cdot} A C_{i\cdot}^\top := \Sigma_i^{\rm DDV}\\
    \mathbb{E}_{x,\epsilon} {\rm Tr}(U\sigma(x)\epsilon \epsilon^\top \sigma(x) U^\top) & = {\rm Tr}(U {\rm diag}(\mathbb{E}_{x} \sigma_1^2(x), \cdots, \mathbb{E}_{x} \sigma_{d_1}^2(x) ) U^\top)\\
    \mathbb{E}_{x,\epsilon} {\rm Tr}(U (W^\top x  - y)\epsilon^\top \sigma(x) U^\top ) & = \mathbb{E}_{x} {\rm Tr}(U (W^\top x  - y)\mathbb{E}_{\epsilon}\epsilon^\top \sigma(x) U^\top ) = 0,
\end{align}
where $\mu_x$ is the mean vector of input variable $x$, and $f_i C_{i\cdot} \mu_x$ is the inner product of the $i$-th row of $C$ and $\mu_x$ multiplied by a scalar $f_i$. This corollary means that the loss function can be written in the following form:
\begin{align}
    L_{\rm VAE}^{\rm DDV}(U, W, C, f)
    =  \frac{1}{2\eta_{\rm dec}^2} \mathbb{E}_{x,\epsilon} \bigg[ \|U W^\top x  - y\|^2 &+ {\rm Tr}(U \Sigma^{\rm DDV} U^\top) 
  + \beta \frac{\eta_{\rm dec}^2}{\eta_{\rm enc}^2} \|W^\top x\|^2 \bigg] \nonumber\\ &+  \frac{\beta}{2} \sum_{i=1}^{d_1} \mathbb{E}_{x} \left(\frac{\sigma_{i}^2(x)}{\eta^2_{\rm enc}} - 1 - \log \frac{\sigma_{i}^2(x)}{\eta_{\rm enc}^2}\right).
\end{align}

What makes the problem analytical intractable is the term $\mathbb{E}_{x} \log(\sigma_i^2 (x)) $. However, we can still obtain some very insightful qualitative results from it. 

The following lemma will help us show that it is always better to have $C=0$. 

\begin{lemma}\label{lemma: equal variance condition}
    For any $C,\ f$, there exists $f'$ such that $\E_x\sigma_i^2(x; C,f) = \E_x\sigma_i^2(x; C=0, f')$.\footnote{Note that we have now explicitly written out $C$ and $f$ to emphasize that $\sigma$ is also a function of $C$ and $f$}.
\end{lemma}
\textit{Proof}. By definition,
\begin{equation}
    \mathbb{E}_{x} \sigma_i^2(x; C, f) = f_i^2 + C_{i\cdot} A C_{i\cdot}^\top,
\end{equation}
and 
\begin{equation}
    \mathbb{E}_{x} \sigma_i^2(x; C=0, f') = (f_i')^2. 
\end{equation}
Now, setting
\begin{equation}
    f_i' = \sqrt{f_i^2 + C_{i\cdot} A C_{i\cdot}^\top}
\end{equation}
is sufficient to make the two equal. $\square$

We can now prove that it is always better to have $C=0$.

\begin{proposition}
    For any $U,\ W,\ C, f,$ there exists $f'$ such that
    \begin{equation}
         L_{\rm VAE}^{\rm DDV}(U, W, C, f) \geq  L_{\rm VAE}^{\rm DDV}(U, W, 0, f').
    \end{equation}
\end{proposition}

\textit{Proof}. Throughout, we let $f'$ equal to the form given by Lemma~\ref{lemma: equal variance condition}.

By Eq.~\eqref{eq: equivalent loss function data dependent z}, the loss function can be written as the sum of a term $L_0$ that depends only on $U,\ W$ and $\Sigma^{\rm DDV}$ and the logarithmic term:
\begin{equation}
    L_{\rm VAE}^{\rm DDV}(U, W, C, f) = L_1(U, W, \Sigma^{\rm DDV}) - \frac{\beta}{2} \E_x\log\sigma_i^2(x).
\end{equation}
However, by Lemma~\eqref{lemma: equal variance condition}, we have
\begin{equation}
    L_{\rm VAE}^{\rm DDV}(U, W, 0, f') = L_1(U, W, \Sigma^{\rm DDV}) - \frac{\beta}{2}\log \E_x\sigma_i^2(x).
\end{equation}
Noting that $-\log \sigma_i^2(x)$ is convex, we have, for any $C$ and $f$,
\begin{equation}
    - \E_x\log \sigma_i^2(x) \geq - \log \E_x\sigma_i^2(x).
\end{equation}
This implies that
\begin{equation}
    L_{\rm VAE}^{\rm DDV}(U, W, C, f) -  L_{\rm VAE}^{\rm DDV}(U, W, 0, f') = - \E_x\log \sigma_i^2(x) + \log \E_x\sigma_i^2(x) \geq 0.
\end{equation}
This completes the proof. $\square$

When $C=0$, the encoder variance becomes data-independent, and the global minimum is thus, again, given by the main results in the main text. This result shows that a learnable data-dependent encoder variance does not have any quantitative difference at the global minimum when compared with the case of a data-independent encoder variance. This result is directly supported by our numerical results in Section~\ref{sec: experiment}, where the experiments are done for the case where the encoder variances are actually learned.

\section{Case of Learnable Decoding Variance $\eta_{\rm dec}^2$}\label{app sec: learnable decoder variance}

We first give an explicit form of the loss function at the global minimum found in Theorem~\ref{thm:linear-vae}.
Using the optimal $U^*, W^*, \Sigma^*$, the analytical formulation of the minimal $L_{\rm VAE}$ can be obtained.
\begin{corollary}~\label{corollary:min_VAE}
    The minimal value of the objective function $L_{\rm VAE}$ is
\begin{align}~\label{eq:optimal-L_VAE-without-LDV}
 \min_{U, W, \Sigma} L_{\rm VAE}(U, W, \Sigma)
 &= \frac{1}{2\eta_{\rm dec}^2} \left[\sum_{i=1}^{d^*} \zeta_i^2 - \sum_{i:\zeta_i^2 > \beta\eta_{\rm dec}^2}^{d_1} \zeta_i^2 \left(1+ \frac{\beta \eta_{\rm dec}^2}{\zeta_i^2}\left(\log\frac{\beta \eta_{\rm dec}^2}{\zeta_i^2} - 1\right) \right)  \right],
\end{align}
where $\zeta_{i}^2$ are sorted in non-increasing order. For convenience, we let $\zeta_{i}^2 = 0$ for $d^* \leq i \leq d_1$ when $d_1 > d^*$.
\end{corollary}
Corollary~\ref{corollary:min_VAE} gives the global minimum of the objective for a fixed decoding variance $\eta_{\rm dec}^2$.
The first summation considers all eigenvalues $\zeta_i^2$ while the second summation considers non-zero first $d_1$ eigenvalues. 

Here, we discuss the VAE with a Learnable Decoding Variance (LDV) $\eta_{\rm dec}^2$.
For shorthand, we denote $\eta_{\rm dec}^2 := s \in(0,\infty)$. When we want to optimize over $s$, we also need to include the partition function, $\frac{d_2}{2}\log s$, of the decoder in the loss $L_{\rm VAE}^{\rm LDV}$. We note that this partition function has been ignored in the main text because $s$ has been treated as a constant for $L_{\rm VAE}$. The loss function $L_{\rm VAE}^{\rm LDV}$ with the optimal $U^*$, $W^*$, and $\Sigma^*$ is thus given by combining Eq.~\eqref{eq:optimal-L_VAE-without-LDV} and the partition function $\frac{d_2}{2}\log s$:
\begin{align}
     G(s) & := L_{\rm VAE}^{\rm LDV}(U=U^*, W=W^*, \Sigma=\Sigma^*, \eta_{\rm dec}^2  = s) \nonumber\\
     & = \frac{1}{2 s} \sum_{i=1}^{d^*} \zeta_i^2 - \frac{1}{2}\sum^{ d_1}_{i:\zeta_i^2> \beta s}  \left[ \frac{\zeta_i^2}{s} + \beta \left(\log  \frac{\beta s}{\zeta_i^2} - 1 \right) \right] + \frac{d_2}{2} \log s .\label{eq: Gs}
\end{align}
Next, we investigate how $\beta$ affects the learnable decoding variance $s$ and identify the optimal $s^*$ under various conditions. Then, we show that, even with a learnable $\eta_{\rm dec}^2$, the specific choice of $\beta$ can lead to or avoid the posterior collapse.

Moreover, for clarity, let $\hat d^*$ be the number of non-zero $\zeta_i^2$ for $1\leq i \leq d^*$, and $\hat d_1$ be the number of non-zero $\zeta_i^2$ for $1\leq i \leq d_1$. It is easy to see that $\hat d_1 \leq \hat d^*$. The loss is
\begin{align}
    G(s) = \frac{1}{2 s} \sum_{i=1}^{\hat d^*} \zeta_i^2 - \frac{1}{2}\sum^{\hat d_1}_{i} F_i(s) + \frac{d_2}{2} \log s,
\end{align}
where 
\begin{equation}
    F_i(s) := \begin{cases} \frac{\zeta_i^2}{s} + \beta \left(\log  \frac{\beta s}{\zeta_i^2} - 1 \right) & \text{when $s < \frac{\zeta_i^2}{\beta}$,}\\
    0 & \text{otherwise}.
    \end{cases}
\end{equation}

\begin{lemma}
    $G(s)$ is differentiable.
\end{lemma}
\begin{proof} It suffices to check that $F_i(s)$ is differentiable on $(0,\infty)$ for all $i$. By definition, $F_i(s)$ is differentiable except at $\zeta_i^2 = \beta s$, and thus it suffices to check its differentiability at $\zeta_i^2/ \beta $.

First of all, $F$ is continuous:
\begin{equation}
    \lim_{s\to (\zeta_i^2/ \beta)^-}F(s) = 0 = \lim_{s\to (\zeta_i^2/ \beta)^+}F(s).
\end{equation}
Then, $F$ is differentiable:
\begin{equation}
    \lim_{s\to (\zeta_i^2/ \beta)^-}F'_i(s) =\lim_{s\to (\zeta_i^2/ \beta)^-} \frac{\beta}{s^2}(s - \zeta_i^2/ \beta) = 0 = \lim_{s\to (\zeta_i^2/ \beta)^+}F'_i(s).
\end{equation}
This finishes the proof.
\end{proof}

Therefore, we only need to check the stationary points and the right limit of $G(s)$ at $0$ and the left limit at $\infty$. We proceed by first considering the monotonicity over intervals defined by the piecewise function and then narrowing down the solution of $G'(s) = 0$ into a specific interval.

Let $s_p := \frac{1}{\beta}\zeta_{p}^2$ for $p = 1, \cdots, \hat d_1$. We define $s_{\hat d_1+1} = 0$ and $s_{0} = \infty$. Because the $\zeta_i$ are listed in non-increasing order, we have
$0 = s_{\hat d_1+1} < s_{\hat d_1} \leq \cdots \leq s_{1} < s_{0} = +\infty$. Then $(0, +\infty) = \bigcup_{p=0}^{\hat d_1} [s_{p+1}, s_p) - \{0\}$ can be decomposed into the union of a set of intervals.
For each interval $[s_{p+1}, s_p)$,
\begin{align}
    G'(s) = \frac{1}{2s^2}\left[ (d_2 - \beta p) s - \sum_{i=p+1}^{d^*} \zeta_i^2 \right],~\label{eq:piecewise-formula}
\end{align}
where we implicitly define $\sum_{i=p+1}^{\hat d^*} \zeta_i^2 := 0$ when $p \geq \hat d^*$.

The following lemma states the number of stationary point of $G(s)$ in an interval $[s_{p+1}, s_p)$.
\begin{lemma}~\label{lemma:interval-stationary}
At each interval $[s_{p+1}, s_p), 0\leq p \leq \hat d_1$, $G'(s)$ has at most one stationary point when $(d_2 - \beta \hat d_1) \neq 0$ or infinite stationary points when $(d_2 - \beta \hat d_1) = 0$ and $p = \hat d_1 = \hat d^*$.
\end{lemma}
\begin{proof}
The existence of stationary points requires $G'(s) = 0$, which is equivalent to $(d_2 - \beta p) s = \sum_{i=p+1}^{d^*} \zeta_i^2$. When $p < \hat d_1$, $ \sum_{i=p+1}^{d^*} \zeta_i^2 > 0$ holds. Therefore, $G'(s) = 0$ has at most one solution. 

When $p = \hat d_1 = \hat d^*$, $G'(s) = 0$ holds only if $(d_2 - \beta \hat d_1) = 0$. Then, $\forall s \in (0, s_{\hat d_1})$ is the stationary point.
\end{proof}
Moreover, by Eq.~\eqref{eq:piecewise-formula}, we have the following corollary.
\begin{corollary}~\label{corollary:endpoints}
If there is a unique stationary point at $[s_{p+1}, s_p)$, $G'(s_p) G'(s_{p+1}) \leq 0 $.
\end{corollary}

The derivative at the endpoints can be computed as
\begin{align}
    G'(s_p) = \frac{\beta^2}{2\zeta_{p}^4} \left[ \left(\frac{d_2}{\beta} - (p-1) \right)\zeta_{p}^2 - \sum_{i=p}^{\hat d^*} \zeta_i^2 \right] = \frac{\beta^2}{2\zeta_{p}^4} \left[ \left(\frac{d_2}{\beta} - p \right)\zeta_{p}^2 - \sum_{i=p+1}^{\hat d^*} \zeta_i^2 \right].
\end{align}

Furthermore, once $G'(s_p)$ is non-negative at some endpoint $s_p$, $G'(s) > 0$ holds over $(s_p, \infty)$.
\begin{lemma}~\label{lemma:v-shape-points}
Let $p, t$ be such that $p \leq d_2 / \beta$ and $t > s_p$. Then, $t^2 G'(t) > s_p^2 G'(s_p)$.
\end{lemma}
\begin{proof}
Let $t \in (s_{q+1}, s_q]$ such that $q+1 \leq p$. We thus have $t > s_{q+1} \geq \cdots \geq s_{p}$.
Because $\frac{d_2}{\beta} - (p-1) > 0$, we have
\begin{align}
    G'(t) & = \frac{1}{2t^2}\left[ (d_2 - \beta q) t - \sum_{i=q+1}^{\hat d^*} \zeta_i^2 \right] = \frac{1}{2t^2}\left[ (d_2 - \beta (p-1)) t - \sum_{i=p}^{q} \beta t - \sum_{i=p}^{\hat d^*} \zeta_i^2 \right]\\
    & > \frac{1}{2t^2}\left[ (d_2 - \beta (p-1)) t - \sum_{i=p}^{\hat d^*} \zeta_i^2 \right] = \frac{s_p^2}{t^2} G'(s_p).
\end{align}
\end{proof}

\begin{proposition}~\label{prop:global-stationary}
$G(s)$ has at most one stationary point over $(0, \infty)$ when $(d_2 - \beta \hat d_1) \neq 0$.
\end{proposition}
\begin{proof}
We prove this by contradiction.
Let $s_a<s_b$ are two stationary points. Lemma~\ref{lemma:interval-stationary} implies these two stationary points are located in two intervals $[s_{p_a+1}, s_{p_a})$ and $[s_{p_b+1}, s_{p_b})$ with $p_a > p_b$. By Corollary~\ref{corollary:endpoints} there exists $s_l\in \{s_{p_a+1}, s_{p_a}\}$ such that $G'(s_l) \geq 0$, and $s_r\in \{s_{p_b+1}, s_{p_b}\}$ such that $G'(s_r) \leq 0$. Noticing that $s_{p_a+1} <s_{p_a} \leq s_{p_b+1} < s_{p_b}$, we conclude that $s_l\leq s_r$. If $s_l < s_r$, it contradicts to Lemma~\ref{lemma:v-shape-points}. If $s_l = s_r = s_{p_a}$, then there is no stationary points in the interval $[s_{p_a+1}, s_{p_a})$, which contradicts to the assumption.
\end{proof}

Now, we check whether $0$ and $\infty$ are minima.
For $s_+ = s_1 + \sum_{i=1}^{\hat d^*} \zeta_i^2 \in [s_1, +\infty)$, we have
\begin{align}
    G'(s_+) = \frac{1}{2s_+^2}\left[ d_2 s_+ - \sum_{i=1}^{\hat d^*} \zeta_i^2 \right] = \frac{1}{2s_+^2}\left[ d_2 s_1 +  (d_2 - 1)\sum_{i=1}^{\hat d^*} \zeta_i^2 \right] > 0,~\label{eq:s+}
\end{align}
which implies that the $\infty$ is not a minimum.

The behavior of $G'(s)$ in $(0, s_{\hat d_1})$ is more complicated.
\begin{align}
    G'(s) = \frac{1}{2 s^2} \left[ \left(\frac{d_2}{\beta} - \hat d_1 \right)s - \sum_{i=\hat d_1 +1}^{\hat d^*} \zeta_i^2 \right],
\end{align}
the sign of which is different for the following three different cases:
\begin{enumerate}
    \item $\hat d_1 = \hat d^*$ and $d_2/\beta - \hat d_1 > 0$;
    \item $\hat d_1 = \hat d^*$ and $d_2/\beta - \hat d_1 = 0$;
    \item $\hat d_1 < \hat d^*$ or $d_2/\beta - \hat d_1 < 0$.
\end{enumerate}

\underline{Case 1: $\hat d_1 = \hat d^*$ and $d_2/\beta - \hat d_1 > 0$.}
When $\hat d_1 = \hat d^*$ and $d_2/\beta - \hat d > 0$, $ \sum_{i=\hat d+1}^{\hat d^*} \zeta_i^2 = 0$ and thus $G'(s) > 0$ in $(0, s_{\hat d_1}]$. By Lemma~\ref{lemma:v-shape-points}, $G'(s) > 0$ for $s > s_{\hat d_1}$. Therefore, $G'(s) > 0$ for $s\in (0, +\infty)$. Then, there is no global minima for $s \in (0,+\infty)$. The loss function $L_{\rm VAE}^{\rm LDV}$ is ill-posed. Even though $s^*$ is converged to $0$, the model is deterministic.

\underline{Case 2: $\hat d_1 = \hat d^*$ and $d_2/\beta - \hat d_1 = 0$.}
When $\hat d_1 = \hat d^*$ and $d_2/\beta - \hat d = 0$, we have $ \sum_{i=\hat d+1}^{d^*} \zeta_i^2 = 0$ and thus $G'(s_{\hat d}) = 0$ for all $s \in (0, s_{\hat d_1})$. $G'(s_{\hat d}) = 0$ also holds by the continuity of $G'(s)$. For any $s > s_{\hat d_1}$, $G'(s) > 0$ by Lemma~\ref{lemma:v-shape-points}. Then the global minima for of $G(s)$ is the entire set of $(0, s_{\hat d_1}]$. In such a case, no posterior collapse happens.

\underline{Case 3: $\hat d_1 < \hat d^*$ or $d_2/\beta - \hat d_1 < 0$.}\footnote{The case where $\hat d_1 < \hat d^*$ and $\beta=1$ is the case discussed in~\citet{lucas2019don} and a variant of the case considered in~\citep{nakajima2015condition}}
The following proposition shows that the global minimum of $G'(s)$ is unique.
\begin{proposition}~\label{prop:global-minimum}
When $\hat d_1 < \hat d^*$ or $d_2/\beta - \hat d_1 < 0$,
$G(s)$ has a unique global minimum, which is the unique stationary point.
\end{proposition}
\begin{proof}
We first prove the existence.
Let $s_- = \min\left\{s_{\hat d_1}, \sum_{i=\hat d_1 +1}^{d^*} \zeta_i^2 / \left(\frac{d_2}{\beta} - \hat d_1\right)\right\}$. Then, 
\begin{align}
    G'(s) = \frac{1}{2 s^2}\left[\left(\frac{d_2}{\beta} - \hat d_1\right)s  -\sum_{i=\hat d_1 +1}^{d^*} \zeta_i^2\right] < 0
\end{align}
holds in $(0, s_-)$. Recall that $G'(s_+) > 0$ in Eq.~\eqref{eq:s+}. Then, 
\begin{align}
    G'(s) = \frac{1}{2 s^2}\left[\frac{d_2}{\beta} s  -\sum_{i=\hat d_1 +1}^{d^*} \zeta_i^2\right] > \frac{1}{2 s^2}\left[\frac{d_2}{\beta} s_+  -\sum_{i=\hat d_1 +1}^{d^*} \zeta_i^2\right] > 0
\end{align}
holds in $(s_+, \infty)$.
Meanwhile, the continuous function $G$ has minima in the closed interval $[s_-, s_+]$. Then, there exists global minima of $G(s)$ in $(0, \infty)$.

We prove the uniqueness by contradiction. Suppose there are two different global minima such that $s_a^* < s_b^*$.
On the one hand, if there are two $p_a < p_b$ such that $s_a^* \in [s_{p_a+1}, s_{p_a})$ and $s_b^* \in [s_{p_b+1}, s_{p_b})$, we have $G'(s_{p_a+1}) > 0$. At the same time, $s_b^*$ is the global minimum implies $G'(s_{p_b+1})\leq 0$, which contradicts to the Lemma~\ref{lemma:v-shape-points}.
On the other hand, if there is a unique $p_a$ such that $s_a^*, s_b^* \in [s_{p_a+1}, s_{p_a})$, $G'(s_a^*) = G'(s_b^*) = 0$, that is $(d_2 / \beta - \hat p_a) s_a^* = (d_2 / \beta - \hat p_a) s_b^*$. This implies $d_2 / \beta - \hat p_a = 0$. Therefore, $G'(s_b^*) = - \frac{1}{2 s_b^{*2}} \sum_{i=\hat p_a +1}^{d^*} \zeta_i^2 = 0$, which contradicts the proposition assumption.

By Proposition~\ref{prop:global-stationary}, the global minimum of differentiable function $G$ is also the stationary point.
\end{proof}

Now, we are ready to find the optimal $s^*$.

\begin{theorem}
When $\hat d_1 < \hat d^*$ or $d_2/\beta - \hat d_1 < 0$,
\begin{itemize}
    \item The optimal decoding variance is $\eta_{\rm dec}^{*2} = \frac{\sum_{i=\hat d_1+1}^{\hat d^*} \zeta_i^2}{d_2 - \beta \hat d_1} \in (0, s_{\hat d_1}) $ if and only if
\begin{align}
    \beta < \frac{d_2 \zeta_{\hat d_1}^2}{d_1\zeta_{\hat d_1}^2 + \sum_{i=\hat d_1+1}^{\hat d^*} \zeta_i^2}.
\end{align}
    \item The optimal decoding variance is $\eta_{\rm dec}^{*2} = \frac{\sum_{i=p+1}^{\hat d^*} \zeta_i^2}{d_2 - \beta p} \in (s_{p+1}, s_p) $, for $1\leq p<\hat d_1$ if and only if
\begin{align}
    \frac{d_2 \zeta_{p+1}^2}{\sum_{i=p+2}^{\hat d^*} \zeta_i^2 + (p+1) \zeta_{p+1}^2} \leq \beta <\frac{d_2 \zeta_{p}^2}{\sum_{i=p+1}^{\hat d^*} \zeta_i^2 + p \zeta_{p}^2}.
\end{align}
    \item  The optimal decoding variance is $\eta_{\rm dec}^{*2}=\frac{1}{d_2} \sum_{i=1}^{\hat d^*}\zeta_i^2 \in (s_1, \infty)$ if and only if
\begin{align}
    \beta \geq \frac{d_2 \zeta_1^2}{\sum_{i=1}^{\hat d^*} \zeta_i^2}.
\end{align}
\end{itemize}
\end{theorem}
\begin{proof}
To ensure $s^* \in (0, s_{\hat d_1})$, then the condition for $\beta$ can be derived by letting $G'(s_{\hat d_1}) > 0$, that is
\begin{align}
    (\frac{d_2}{\beta} - \hat d_1) \zeta_{\hat d_1}^2 - \sum_{i=\hat d_1+1}^{\hat d^*} \zeta_i^2 > 0.
\end{align}
The optimal $s$ is solved by $G'(s) = 0$, that is
\begin{align}
    (d_2 - \beta \hat d_1) s - \sum_{i=\hat d_1+1}^{\hat d^*} \zeta_i^2 = 0.
\end{align}

To ensure $s^* \in [s_{p+1}, s_{p})$, then the condition for $\beta$ can be derived by letting $G'(s_{p}) > 0 \geq G'(s_{p+1})$, that is
\begin{align}
    (\frac{d_2}{\beta} - p) \zeta_{p}^2 - \sum_{i=p+1}^{d^*} \zeta_i^2 > 0 \geq (\frac{d_2}{\beta} - (p+1)) \zeta_{p+1}^2 - \sum_{i=p+2}^{d^*} \zeta_i^2.
\end{align}
Then the optimal $s^*$ is solved by $G'(s) = 0$, that is
\begin{align}
     (d_2 - \beta p) s - \sum_{i=p+1}^{\hat d^*} \zeta_i^2 = 0.
\end{align}

To ensure $s^* \in [s_1, +\infty)$, that is, $G'(s_{1})\leq 0$. The condition for $\beta$ can be derived by solving
\begin{align}
\frac{d_2}{\beta}\zeta_{1}^2- \sum_{i=1}^{d^*} \zeta_i^2 \leq 0.
\end{align}
Optimal $s^*$ can be found by solving $G'(s) = 0$, that is
\begin{align}
    d_2 s^* - \sum_{i=1}^{\hat d^*} \zeta_i^2 = 0.
\end{align}
\end{proof}

\begin{remark}
By Theorem~\ref{thm:linear-vae}, the posterior collapse for an eigenvalue $\zeta_i^2$ happens when $\beta \eta_{\rm dec}^2 \geq \zeta_i^2$, which is equivalent to $s \geq s_p$ for $1\leq p\leq \hat d_1$. Therefore, different types of posterior collapse are related to the following conditions of $s^*$
\begin{compactitem}
    \item No collapse: $s\in (0, s_{\hat d_1})$;
    \item Partial collapse: $s\in [s_{p+1}, s_p)$ for $1\leq p\leq \hat d_1$;
    \item Complete collapse: $s\in [s_{1}, \infty)$.
\end{compactitem}
Notably, our result shows that the linear VAE with learnable decoding variance does not suffice to lead to no collapse. For an arbitrary choice of $\zeta_i^2$, the condition for no posterior collapse is $\beta \in (0, d_2/\hat d^*)$.
When $d_2=d_0=d^*$, $d_1=\hat d_1$, and $\beta=1$, the third case reduces to the result of \cite{lucas2019don}. However, posterior collapse also happens in this case. For example, when $\zeta_1^2 = ... = \zeta_{\hat d^*}^2$, the condition for the complete collapse of the model in~\citet{lucas2019don} is $[1, \infty)$, which covers the current choice of $\beta = 1$.
\end{remark}

\begin{table}[t]
\centering
\caption{The effect of $\beta$ for posterior collapse with learnable decoding variance}~\label{tb:beta-conditions}
\begin{tabular}{llp{4cm}l}
\toprule
Dimension & $\beta$ Range & Posterior Collapse & $\eta_{\rm dec}^{*2}$ \\
\midrule
$\hat d_1 = \hat d^*$  &$(0, d_2 / \hat d_1)$  & NA & NA \\
$\hat d_1 = \hat d^*$  &$\{ d_2 / \hat d_1\} $ & No collapse or $\zeta_{\hat d_1}^2$ only & $(0, s_{\hat d_1}]$\\
$\hat d_1 < \hat d^*$  & $\left(0, \frac{d_2 \zeta_{\hat d_1}^2}{d_1\zeta_{\hat d_1}^2 + \sum_{i=\hat d_1+1}^{\hat d^*} \zeta_i^2}\right)$ & No collapse & $\frac{\sum_{i=\hat d_1+1}^{\hat d^*} \zeta_i^2}{d_2 - \beta \hat d_1}$ \\
$\hat d_1 \leq \hat d^*$  & $\left[\frac{d_2 \zeta_{p+1}^2}{\sum_{i=p+2}^{\hat d_1} \zeta_i^2 + (p+1) \zeta_{p+1}^2}, \frac{d_2 \zeta_{p}^2}{\sum_{i=p+1}^{\hat d_1} \zeta_i^2 + p \zeta_{p}^2}\right)$ & Partial collapse except the first $p$ modes, $ 1\leq p< \hat d_1$ & $\frac{\sum_{i=p+1}^{\hat d^*} \zeta_i^2}{d_2 - \beta p}$ \\
$\hat d_1 \leq \hat d^*$  & $\left[ \frac{d_2 \zeta_1^2}{\sum_{i=1}^{\hat d^*} \zeta_i^2}, +\infty\right)$ & Complete collapse & $\frac{1}{d_2} \sum_{i=1}^{\hat d^*}\zeta_i^2$\\
\bottomrule
\end{tabular}
\end{table}

To summarize, Table~\ref{tb:beta-conditions} concludes five situations for posterior collapse under various conditions.

\clearpage
\section{Proofs}\label{app sec: proofs}

\subsection{Proof of Proposition~\ref{proposition1}}
\begin{proof}
Minimizing $L_{\rm VAE}(U, W)$ in Eq.~\eqref{eq:general-vae-loss} is equivalent to the following minimization problem
\begin{align}
    \min_{U, W} \mathbb{E}_{x} \|U W^\top x - y\|^2 +  {\rm Tr}(U \Sigma U^\top) + \beta \frac{\eta_{\rm dec}^2}{\eta_{\rm enc}^2} {\rm Tr}(W^\top A W).
\end{align}
It is assumed that $\tilde x := \Phi^{-\frac{1}{2}} P_A^\top x$, and $x := P_A \Phi^{\frac{1}{2}} \tilde x$. By defining $V := \Phi^{\frac{1}{2}}P_A^\top W$, we obtain
\begin{align}
    & \mathbb{E}_{x} \|U W^\top x - y\|^2 +  {\rm Tr}(U \Sigma U^\top) + \beta \frac{\eta_{\rm dec}^2}{\eta_{\rm enc}^2} {\rm Tr}(W^\top A W)\\
    = &\mathbb{E}_{x} \|U W^\top P_A \Phi^{\frac{1}{2}} \tilde x - y\|^2 + {\rm Tr}(U \Sigma U^\top) + \beta \frac{\eta_{\rm dec}^2}{\eta_{\rm enc}^2} {\rm Tr}(W^\top P_A \Phi P_A^\top W)\\
    = &\mathbb{E}_{x} \|U V \tilde x - y\|^2 + {\rm Tr}(U \Sigma U^\top) + \beta \frac{\eta_{\rm dec}^2}{\eta_{\rm enc}^2} \| V \|_F^2\\
    = & {\rm Tr} (U^\top U V^\top V - 2 U^\top ZV) + \mathbb{E}_{\tilde x} [y^\top y] + {\rm Tr}(U \Sigma U^\top) + \beta \frac{\eta_{\rm dec}^2}{\eta_{\rm enc}^2} {\rm Tr} (V^\top V)\\
    = & \|U V^\top - Z\|^2_F + {\rm Tr}(U \Sigma U^\top) + \beta \frac{\eta_{\rm dec}^2}{\eta_{\rm enc}^2} \|V \|^2_F  - \|Z\|^2_F + \mathbb{E}_{\tilde x} [y^\top y]\\
    = & \|U V^\top - Z\|^2_F + {\rm Tr}(U \Sigma U^\top) + \beta \frac{\eta_{\rm dec}^2}{\eta_{\rm enc}^2} \|V \|^2_F,
\end{align}
where we have used the relation $\E[\tilde{x}\tilde{x}^{T}] = I$ and $\|Z\|^2_F = \mathbb{E}_{\tilde x} [y^\top y]$. Thus, the desired $(U, V)$ can be obtained from minimizing $L(U,V) = \|U^\top V - Z\|^2_F + {\rm Tr}(U \Sigma U^\top) + \beta \frac{\eta_{\rm dec}^2}{\eta_{\rm enc}^2} \|V \|^2_F $. This finishes the proof.
\end{proof}

\subsection{Proof of Proposition~\ref{proposition:min-Luv}}
\begin{proof}
One of the necessary conditions for the global minimum is the zero gradient of $L(U, V)$. We then find the global minimum under the zero gradient condition. Consider
\begin{align}
    \frac{1}{2}\frac{\partial L(U, V)}{\partial V} &= VU^\top U - Z^\top U + \beta \frac{\eta_{\rm dec}^2}{\eta_{\rm enc}^2} V = 0,
\end{align}
which implies
\begin{align}
    & V = Z^\top U \left[\beta \frac{\eta_{\rm dec}^2}{\eta_{\rm enc}^2} I + U^\top U\right]^{-1}.~\label{eq:V_by_U}
\end{align}
Plugging Eq.~\eqref{eq:V_by_U} into the objective in Eq.~\eqref{eq:LVAE_is_RSVD}, we have
\begin{align}
    L(U, V) = & {\rm Tr}\left[\left(U^\top U + \beta \frac{\eta_{\rm dec}^2}{\eta_{\rm enc}^2} I\right) V^\top V - 2 U^\top Z V\right] + {\rm Tr}(U \Sigma U^\top) + \|Z\|^2_F\\
    = & {\rm Tr}(U \Sigma U^\top) - {\rm Tr} (U^\top Z V) + \|Z\|^2_F\\
    = & \underbrace{{\rm Tr}(U \Sigma U^\top) - {\rm Tr} \left[ U^\top Z Z^\top U \left(\beta \frac{\eta_{\rm dec}^2}{\eta_{\rm enc}^2} I + U^\top U\right )^{-1}\right]}_{:= J} + \|Z\|^2_F.~\label{eq:J}
\end{align}

Consider the SVD of matrix $U = Q\Lambda P$ where $Q\in \mathbb{R}^{d_2\times d_2}$ and $P\in \mathbb{R}^{d_1\times d_1}$ are orthogonal matrices, $\Lambda \in \mathbb{R}^{d_2\times d_1}$ is the rectangular diagonal matrix. Meanwhile, consider
\begin{align}
\left(\beta\frac{ \eta_{\rm dec}^2}{\eta_{\rm enc}^2} I + P^\top \Lambda^\top \Lambda P\right)^{-1} = P^\top \left(\beta\frac{ \eta_{\rm dec}^2}{\eta_{\rm enc}^2} I + \Lambda^\top \Lambda \right)^{-1} P.
\end{align}
Let diagonal matrix $\Gamma = \beta\frac{ \eta_{\rm dec}^2}{\eta_{\rm enc}^2} I + \Lambda^\top \Lambda$. 
Recall the SVD of $Z = F \Sigma_Z G$,
then the Eq.~\eqref{eq:J} is rewritten as
\begin{align}
J = {\rm Tr} \left[  \Lambda^\top \Lambda \Sigma\right] - {\rm Tr} \left[ (Q \Lambda \Gamma^{-1} \Lambda^\top Q^\top) ( F \Sigma_Z\Sigma_Z^\top F^\top )\right].
\end{align}
We note that $\Lambda\Gamma \Lambda^\top $ and $\Sigma_Z \Sigma_Z^\top$ are square diagonal matrices in $\mathbb{R}^{d_2 \times d_2}$. Since $\Sigma_Z \in \mathbb{R}^{d_2\times d_0}$ and there are only $\min(d_0, d_2)$ non-zero values, i.e., $\zeta_i, i = 1, ..., \min(d_0, d_2)$. We denote $\zeta_i = 0$ for $\min(d_0, d_2) < i \leq d_1$ if $d_1 > \min(d_0, d_2)$ for convenience.
By von Neumann's Trace Inequality~\citep{von1962some}, the trace of the product of two real symmetric matrices is upper bounded by the sum of the product of their decreasing eigenvalues, specifically,
\begin{align}
    {\rm Tr} \left[ (Q \Lambda \Gamma^{-1} \Lambda^\top Q^\top) ( F \Sigma_Z\Sigma_Z^\top F^\top )\right] \leq {\rm Tr} \left[ \Lambda \Gamma^{-1} \Lambda^\top \Sigma_Z\Sigma_Z^\top \right].
\end{align}
The equality holds if and only if $Q = F$.
Then the lower bound of $J$ is achieved when optimal $Q^* = F$.
\begin{align}
    J \geq {\rm Tr} \left[  \Lambda^\top \Lambda \Sigma\right] - {\rm Tr} \left[ \Lambda \Gamma^{-1} \Lambda^\top \Sigma_Z\Sigma_Z^\top \right] = \underbrace{\sum_{i=1}^{d_1}  \sigma_i^2\lambda_i^2  - \frac{\zeta^2_i \eta_{\rm enc}^2 \lambda_i^2 }{ \beta \eta_{\rm dec}^2 + \eta_{\rm enc}^2\lambda_i^2 }}_{:=J_{Q^*}}.
\end{align}
$J_{Q^*}$ can be further minimized over all $\lambda_i$. The optimal $\lambda_i^*$ can be determined by setting the corresponding gradients to zero. 
Consider $t_i = \lambda_i^2 \geq 0$,
\begin{align}
    \frac{\partial J_{Q^*}}{\partial t_i}
    = & \frac{\partial}{\partial t_i}\left[ \sigma_i^2 t_i  - \frac{\zeta^2_i \eta_{\rm enc}^2 t_i }{ \beta \eta_{\rm dec}^2 + \eta_{\rm enc}^2 t_i } \right]\\
    = & \sigma_i^2 - \frac{\zeta^2_i \eta_{\rm enc}^2 (\beta \eta_{\rm dec}^2 + \eta_{\rm enc}^2 t_i) - \eta_{\rm enc}^2 \zeta^2_i \eta_{\rm enc}^2 t_i }{(\beta \eta_{\rm dec}^2 + \eta_{\rm enc}^2 t_i)^2}\\
    = & \sigma_i^2 - \frac{\zeta^2_i \eta_{\rm enc}^2 \beta \eta_{\rm dec}^2}{(\beta \eta_{\rm dec}^2 + \eta_{\rm enc}^2 t_i)^2}\\
    = & \frac{\sigma_i^2 (\beta \eta_{\rm dec}^2 + \eta_{\rm enc}^2 t_i)^2 - \zeta^2_i \eta_{\rm enc}^2 \beta \eta_{\rm dec}^2}{(\beta \eta_{\rm dec}^2 + \eta_{\rm enc}^2 t_i)^2} = 0.\label{eq:condition_lambda}
\end{align}
Two solutions of the Eq.~\eqref{eq:condition_lambda} are
\begin{align}
    t_i^{(1)} & = \frac{\sqrt{\beta} \eta_{\rm dec}}{\sigma_i \eta_{\rm enc}} \left(\zeta_i - \frac{\sqrt{\beta}\sigma_i \eta_{\rm dec}}{\eta_{\rm enc}}\right),\\
    t_i^{(2)} & = \frac{\sqrt{\beta} \eta_{\rm dec}}{\sigma_i \eta_{\rm enc}} \left(- \zeta_i - \frac{\sqrt{\beta}\sigma_i \eta_{\rm dec}}{\eta_{\rm enc}}\right).
\end{align}
We see that $t_i \geq 0 > t_i^{(2)}$, then the monotonicity of $J_Q^*$ with respect to $t_i$ over $(0, +\infty)$ only depends on $t_i^{(1)}$. Here are two situations:
\underline{(1) $t_i^{(1)} \leq 0$:} $\frac{\partial J_{Q^*}}{\partial t_i} \geq 0$, then $J_Q^*$ increases monotonically with $t_i$. Then the optimal $t_i^* = 0$. \underline{(1) $t_i^{(1)} > 0$:} $\frac{\partial J_{Q^*}}{\partial t_i} > 0$ when $t_i > t_i^{(1)}$ and $\frac{\partial J_{Q^*}}{\partial t_i} < 0$ when $t_i < t_i^{(1)}$. Then the optimal $t_i^* = t_i^{(1)}$.
Therefore, optimal $\lambda_i^*$ is summarized by the two situations above with $\lambda_i^{*2} = t_i^*$

\begin{align}
    \lambda_i^* = \sqrt{\max\left(0,\frac{\sqrt{\beta} \eta_{\rm dec}}{\sigma_i \eta_{\rm enc}} \left(\zeta_i - \frac{\sqrt{\beta}\sigma_i \eta_{\rm dec}}{\eta_{\rm enc}}\right)\right)}, i = 1, ..., d_1.
\end{align}
As a result, $U = Q^* \Lambda^* P $ where $P$ is an arbitrary orthogonal matrix in $\mathbb{R}^{d_1\times d_1}$.
The optimal $V^*$ can also be determined by Eq.~\eqref{eq:V_by_U}
\begin{align}
    V^* = \bar G \Theta P,
\end{align}
where $\bar G = [g_1, ..., g_{d_1}]$, $\Theta = {\rm diag}(\theta_1, ..., \theta_{d_1})$ where $\theta_i = \sqrt{\max\left(0,\frac{\sigma_i\eta_{\rm enc}}{\sqrt{\beta} \eta_{\rm dec}} \left(\zeta_i - \frac{\sqrt{\beta}\sigma_i \eta_{\rm dec}}{\eta_{\rm enc}}\right)\right)}$.
\end{proof}

\subsection{Proof of Corollary~\ref{corollary:min-Luv-val}}
\begin{proof}
The minimum value can be obtained by plugging in the optimal 
\begin{align}
    \lambda_i^* = \sqrt{\max\left(0,\frac{\sqrt{\beta} \eta_{\rm dec}}{\sigma_i \eta_{\rm enc}} \left(\zeta_i - \frac{\sqrt{\beta}\sigma_i \eta_{\rm dec}}{\eta_{\rm enc}}\right)\right)},
\end{align} 
into the lower bound of $L(u, v)$, i.e., 
\begin{align}
    L(U, V) \geq \min_{U,V} L(U, V) = \|Z\|_F^2 + J_{Q^*} = \sum_{i=1}^{d_1} \zeta_i^2 +  \sigma_i^2\lambda_i^{*2}  - \frac{\zeta^2_i \eta_{\rm enc}^2 \lambda_i^{*2} }{ \beta \eta_{\rm dec}^2 + \eta_{\rm enc}^2\lambda_i^{*2} }.
\end{align}
\end{proof}

\subsection{Proof of Proposition~\ref{proposition:min-sigma-enc}}
\begin{proof}
The optimal $\sigma_i$ can be determined by
\begin{align}
    \sigma_i^* = \argmin_{\sigma>0} l_i(\sigma) = \argmin_{\sigma>0} \zeta_i^2 - \left(\zeta_i - \frac{\sqrt{\beta}\sigma_i \eta_{\rm dec}}{\eta_{\rm enc}}\right)^2 \mathbbm{1}_{\zeta_i > \frac{\sqrt{\beta}\sigma_i \eta_{\rm dec}}{\eta_{\rm enc}}} + \beta \eta_{\rm dec}^2 \left(\frac{\sigma_i^2}{\eta^2_{\rm enc}} - 1 - \log \frac{\sigma_i^2}{\eta_{\rm enc}^2}\right).
\end{align}

The gradient of $l_i(\sigma)$ reads,
\begin{align}
    l_i'(\sigma) & = \mathbbm{1}_{\zeta_i > \frac{\sqrt{\beta}\eta_{\rm dec}}{\eta_{\rm enc}}\sigma} \left(\zeta_i - \frac{\sqrt{\beta}\eta_{\rm dec}}{\eta_{\rm enc}} \sigma\right) \frac{\sqrt{\beta}\eta_{\rm dec}}{\eta_{\rm enc}} + \frac{\beta \eta_{\rm dec}^2}{\eta^2_{\rm enc}} \left(\sigma - \frac{\eta^2_{\rm enc}}{\sigma}\right)
\end{align}
Since $l'(\sigma)$ is a increasing function, $l'(\sigma_-) < 0$ when $\sigma_- = \frac{1}{2}\min( \frac{\zeta_i \eta_{\rm enc}}{\sqrt{\beta}\eta_{\rm dec}}, \frac{\sqrt{\beta}\eta_{\rm enc}\eta_{\rm dec}}{\zeta_i} )$, and $l'(\sigma_+) > 0$ when $\sigma_- = 2\max( \frac{\zeta_i \eta_{\rm enc}}{\sqrt{\beta}\eta_{\rm dec}},  2\eta_{\rm enc} )$. Then the minimal value of $l(\cdot)$ is determined when $l'(\sigma) = 0$, that is,
\begin{align}
    \mathbbm{1}_{\sigma < \zeta_i \frac{\eta_{\rm enc}}{\sqrt{\beta}\eta_{\rm dec}}} \frac{\sqrt{\beta}\eta_{\rm dec}}{\eta_{\rm enc}} \zeta_i + \left[1-\mathbbm{1}_{\sigma < \zeta_i \frac{\eta_{\rm enc}}{\sqrt{\beta}\eta_{\rm dec}}} \right] \frac{\beta \eta_{\rm dec}^2}{\eta^2_{\rm enc}} \sigma = \frac{\beta \eta_{\rm dec}^2}{\sigma}.\label{eq:optimal-sigma-enc}
\end{align}
The LHS of Equation~\eqref{eq:optimal-sigma-enc} is a non-decreasing function while the RHS is decreasing function. Then we claim there is a unique solution $\sigma^*$ of Equation~\eqref{eq:optimal-sigma-enc}. The solution breaks down into two situations

\noindent\textbf{Case 1: $\sigma^* < \zeta_i \frac{\eta_{\rm enc}}{\sqrt{\beta}\eta_{\rm dec}}$}

In this case, we have
\begin{align}
\frac{\sqrt{\beta}\eta_{\rm dec}\zeta_i}{\eta_{\rm enc}} = \frac{\beta \eta_{\rm dec}^2}{\sigma}.
\end{align}
Then
\begin{align}
    \sigma^* = \frac{\sqrt{\beta}\eta_{\rm dec}\eta_{\rm enc}}{\zeta_i}.
\end{align}
This solution holds if and only if the following condition holds
\begin{align}
    \frac{\sqrt{\beta}\eta_{\rm dec}\eta_{\rm enc}}{\zeta_i} < \zeta_i \frac{\eta_{\rm enc}}{\sqrt{\beta}\eta_{\rm dec}} \Leftrightarrow \beta \eta_{dec}^2 < \zeta_i^2.
\end{align}

\noindent\textbf{Case 2: $\sigma^* \geq \zeta_i \frac{\eta_{\rm enc}}{\sqrt{\beta}\eta_{\rm dec}}$}

In this case, we have
\begin{align}
\frac{\beta \eta_{\rm dec}^2}{\eta^2_{\rm enc}} \sigma = \frac{\beta \eta_{\rm dec}^2}{\sigma}.
\end{align}
Then
\begin{align}
    \sigma^* = \eta_{\rm enc}.
\end{align}
This solution holds when
\begin{align}
\eta_{\rm enc} > \zeta_i \frac{\eta_{\rm enc}}{\sqrt{\beta}\eta_{\rm dec}} \Leftrightarrow \beta \eta_{dec}^2 \geq \zeta_i^2.
\end{align}
It is easy to check that the two cases above cover all solutions.
\end{proof}

\subsection{Proof of Theorem~\ref{theo: saddle point condition}}
\begin{proof}
The first-order derivatives of $L$ are
\begin{align}
    \frac{\partial L(U, V)}{\partial U} & = 2 \left(UV^\top V - ZV + U\Sigma\right)\\
    \frac{\partial L(U, V)}{\partial V} & = 2 \left(VU^\top U - Z^\top U + \beta \frac{\eta_{\rm dec}^2}{\eta_{\rm enc}^2} V \right)
\end{align}

Then the second-order derivatives of $L$ are
\begin{align}
\frac{\partial^2 L(U, V)}{\partial U_{rs} \partial U_{pq}} & = 2\delta_{pr}\left( \sum_{k=1}^{d_0} V_{ks} V_{kq} + \delta_{qs} \sigma_q^2 \right)\\
\frac{\partial^2 L(U, V)}{\partial V_{rs} \partial U_{pq}} & = 2 \left[ U_{ps} V_{rq} + \left(\sum_{j=1}^{d_1} U_{pj} V_{rj} - Z_{pr} \right) \delta_{qs} \right]\\
\frac{\partial^2 L(U, V)}{\partial U_{rs} \partial V_{pq}} & = 2 \left[ V_{ps} U_{rq} + \left(\sum_{j=1}^{d_1} U_{rj} V_{pj} - Z_{rp} \right)\delta_{qs}\right]\\
\frac{\partial^2 L(U, V)}{\partial V_{rs} \partial V_{pq}} & = 2 \delta_{pr} \left( \sum_{i=1}^{d_2} U_{is} U_{iq} + \beta \frac{\eta_{\rm dec}^2}{\eta_{\rm enc}^2} \delta_{qs} \right).
\end{align}

Letting $U = 0$ and $V = 0$

\begin{align}
\frac{\partial^2 L(U, V)}{\partial U_{rs} \partial U_{pq}} \bigg\rvert_{U=0, V=0} & = 2\delta_{pr}\delta_{qs} \sigma_q^2\\
\frac{\partial^2 L(U, V)}{\partial V_{rs} \partial U_{pq}} \bigg\rvert_{U=0, V=0} & = - 2  Z_{pr} \delta_{qs}\\
\frac{\partial^2 L(U, V)}{\partial U_{rs} \partial V_{pq}} \bigg\rvert_{U=0, V=0} & = - 2 Z_{rp} \delta_{qs}\\
\frac{\partial^2 L(U, V)}{\partial V_{rs} \partial V_{pq}} \bigg\rvert_{U=0, V=0} & = 2 \beta \frac{\eta_{\rm dec}^2}{\eta_{\rm enc}^2}  \delta_{pr} \delta_{qs}.
\end{align}

Then we consider the quadratic form at $U=0, V=0$.  Consider  $\Delta U$ and $\Delta V$ as the perturbation of $U$ and $V$. Then the quadratic form reads

\begin{align}
LQ(\Delta U, \Delta V) = &\left[\sum_{pqrs}  \frac{\partial^2 L(U, V)}{\partial U_{rs} \partial U_{pq}} \bigg\rvert_{U=0, V=0}  \Delta U_{rs} \Delta U_{pq} + \sum_{pqrs} \frac{\partial^2 L(U, V)}{\partial V_{rs} \partial U_{pq}} \bigg\rvert_{U=0, V=0} \Delta V_{rs} \Delta U_{pq} \right. \\ 
& \left. + \sum_{pqrs} \frac{\partial^2 L(U, V)}{\partial U_{rs} \partial V_{pq}} \bigg\rvert_{U=0, V=0} \Delta U_{rs} \Delta V_{pq} + \sum_{pqrs} \frac{\partial^2 L(U, V)}{\partial V_{rs} \partial V_{pq}} \bigg\rvert_{U=0, V=0} \Delta V_{rs} \Delta V_{pq} \right]\\
= & 2 \sum_{pq} \sigma_q^2 \Delta U_{pq}^2 -2 \sum_{pqr} Z_{pr}\Delta U_{pq}\Delta V_{rq} - 2\sum_{pqr} Z_{rp} \Delta U_{rq} \Delta V_{pq} + \sum_{pq} 2\beta \frac{\eta_{\rm dec}^2}{\eta_{\rm enc}^2} \Delta V_{pq}^2\\
= & 2 \left[{\rm Tr} (\Delta U\Sigma (\Delta U)^\top) - 2 {\rm Tr} (Z \Delta V (\Delta U)^\top ) +  \beta \frac{\eta_{\rm dec}^2}{\eta_{\rm enc}^2} {\rm Tr} (\Delta V (\Delta V)^\top) \right]
\end{align}

It suffices to consider the case $\|X\|_F^2 = 1$. Let $\alpha^2= \|\Delta U\|_F^2$ and $\|\Delta V\|_F^2 = 1-\alpha^2$. Then let $u = \Delta U / \alpha$ and $v = \Delta V / \sqrt{1-\alpha^2}$ be the normalized matrix. Plugging in, we obtain
\begin{align}
    LQ(\Delta U, \Delta V) &\propto \sigma^2 ||U||^2  - 2 {\rm Tr} (ZVU^\top ) +  \beta \frac{\eta_{\rm dec}^2}{\eta_{\rm enc}^2} ||V||^2\\
    &= \sigma^2 \alpha^2 - 2 \sqrt{\alpha^2(1-\alpha^2)}{\rm Tr} (u^\top Zv ) +  \beta \frac{\eta_{\rm dec}^2}{\eta_{\rm enc}^2} (1-\alpha^2), 
\end{align}
where we assume $\Sigma = \sigma I$, according to the Theorem~\ref{thm:linear-vae}. Apparently, for any fixed $\alpha$, the middle term is minimized if $u$ is the left eigenvector of $Z$ corresponding to the largest singular value of $Z$, and $v$ is the corresponding right eigenvector. This choice gives
\begin{equation}
    LQ(\Delta U, \Delta V) \propto \sigma^2 \alpha^2 - 2 \sqrt{\alpha^2(1-\alpha^2)} \zeta_{\rm max} +  \beta \frac{\eta_{\rm dec}^2}{\eta_{\rm enc}^2} (1-\alpha^2).
\end{equation}
Minimizing over $\alpha$ shows that
\begin{equation}
    \min_\alpha LQ(\Delta U, \Delta V) \propto \sigma^2 + \beta \frac{\eta_{\rm dec}^2}{\eta_{\rm enc}^2} - \sqrt{\left(\sigma^2 - \beta \frac{\eta_{\rm dec}^2}{\eta_{\rm enc}^2}\right)^2 + 4 \zeta_{\rm max}^2 }
\end{equation}
which is nonnegative if and only if $\zeta_{\rm max}^2 \geq \sigma^2 \beta \frac{\eta_{\rm dec}^2}{\eta_{\rm enc}^2}$. Namely, $\sigma^2 \beta \frac{\eta_{\rm dec}^2}{\eta_{\rm enc}^2} - \zeta_{\rm max}^2 <0$ implies that the origin is a saddle point. Meanwhile, $\sigma^2 \beta \frac{\eta_{\rm dec}^2}{\eta_{\rm enc}^2} - \zeta_{\rm max}^2 >0$ implies that the origin is a local minimum. Notice that this condition coincides with the condition that the origin is a global minimum. Therefore, the origin is the global minimum if and only if the Hessian at the origin is PSD. This finishes the proof.
\end{proof}


\end{document}